\documentclass[letterpaper,journal]{IEEEtran}
\usepackage{amsmath,amsfonts,amssymb}
\usepackage{algorithmic}
\usepackage{algorithm}
\usepackage{array}
\usepackage[caption=false,font=normalsize,labelfont=sf,textfont=sf]{subfig}
\usepackage{enumerate}
\usepackage{textcomp}
\usepackage{stfloats}
\usepackage{url}
\usepackage{verbatim}
\usepackage{graphicx}
\graphicspath{{figures/}}
\usepackage{cite}
\usepackage{booktabs}

\newtheorem{proposition}{Proposition}

\begin{document}

\title{Koopman Dreamer: Spectrally Constrained Latent Dynamics for Stable World-Model Imagination}

\author{Jiaqi Li, Xinglong Zhang, Haibin Xie, Wei Jiang, Yixing Lan, Wei Pan, and Xin Xu\thanks{Jiaqi Li, Xinglong Zhang, Haibin Xie, Wei Jiang, Yixing Lan, and Xin Xu are with the College of Intelligence Science and Technology, National University of Defense Technology, Changsha, China. Wei Pan is with the School of Engineering, Newcastle University, Newcastle-upon-Tyne, UK. Corresponding authors: Xinglong Zhang.} }

\maketitle

\begin{abstract}
Latent world models improve sample efficiency in continuous control by optimizing policies over imagined latent trajectories, but common neural transitions offer limited direct control over modal persistence and error accumulation in long rollouts. We propose Koopman Dreamer, a Dreamer-style world model with a spectrally constrained deterministic latent dynamics core. Its Koopman-inspired backbone uses two-dimensional rotation--scaling blocks with bounded radii to represent damping, rotation, and near-periodic modes. Linear and low-rank bilinear action terms capture global and state-dependent control effects, while stochastic-state modulation supplies local correction information. To reduce the mismatch between posterior-conditioned training and prior-only imagination, the model combines posterior-conditioned EMA teacher targets with one-step consistency, multi-step rollout, and open-loop observation-prediction objectives. We further derive a multi-step rollout-error bound that separates amplification by the spectral backbone and bilinear interaction from the additive effects of stochastic-state mismatch and modeling residuals, clarifying the trade-off between error attenuation and long-term information retention. Experimental results on proprioceptive continuous-control tasks from the DeepMind Control Suite and UAV-LiDAR autonomous navigation demonstrate that Koopman Dreamer improves the stability of long-horizon latent rollouts and achieves stronger closed-loop control performance on tasks that rely on high-quality multi-step imagination.

\end{abstract}

\begingroup
\renewcommand{\abstractname}{Note to Practitioners}
\begin{abstract}
World-model-based controllers can reduce costly interaction with physical systems by training policies on trajectories imagined through a learned dynamics model, but one-step prediction errors may accumulate and distort policy updates or control decisions. Koopman Dreamer replaces the conventional deterministic latent transition in a Dreamer-style controller with a spectrally structured transition whose modal persistence can be explicitly bounded, while retaining the observation, reward, continuation, and actor--critic components. It is intended for continuous-control applications that rely on recursive multi-step prediction, including robotic motion and autonomous navigation. In simulation, it outperformed DreamerV3 on eight of nine proprioceptive tasks and increased UAV-LiDAR target success from 53.8\% to 73.8\%. Because the most contractive spectral setting was not always optimal, practitioners should tune the spectral range using both decoded observation-prediction accuracy and closed-loop performance. Although validation is currently simulation-based, the structured transition can be adapted to other learned-model control and planning frameworks; physical deployment still requires task-specific robustness evaluation and safety assurance.
\end{abstract}
\endgroup

\section{Introduction}\label{i.-introduction}
\IEEEPARstart{M}{odel-based} reinforcement learning improves the efficiency of interaction with real environments by learning environment dynamics \cite{sutton2018rl,chua2018pets,janner2019mbpo}. Instead of relying entirely on trial and error in the environment, a world model can generate future trajectories under learned dynamics and use these imagined trajectories to train the policy and value function \cite{ha2018world,hafner2019planet,hafner2020dreamer,hansen2024tdmpc2}. Latent world models are particularly important for continuous control: they compress observation histories into compact hidden states and perform state prediction, reward prediction, and policy optimization in latent space, thereby avoiding direct long-horizon dynamics modeling in the original high-dimensional observation space \cite{hafner2019planet,hafner2020dreamer,hafner2021dreamerv2,hafner2025dreamerv3}.

A complete latent world model typically consists of observation representation learning, latent-state transition, prediction heads, and latent-space policy optimization \cite{hafner2020dreamer,hafner2025dreamerv3}. The observation encoder maps task observations into compact representations used to infer latent states, while the dynamics module recursively propagates these states over time. Prediction heads decode the resulting latent states to reconstruct observations and estimate rewards and episode-continuation flags, and the policy and value function are optimized on trajectories imagined through the learned dynamics. Consequently, accurate and stable long-horizon latent transitions are important not only for future observation prediction, but also for imagined-return estimation and policy learning.

Recent latent world models commonly parameterize their transitions with flexible recurrent or feed-forward neural networks \cite{hafner2019planet,hafner2020dreamer,hafner2025dreamerv3,hansen2024tdmpc2}. These models are highly expressive, but their objectives do not usually expose or directly parameterize the modal properties of deterministic latent propagation. Consequently, contraction, persistence, and oscillation emerge implicitly from data and optimization. This is not evidence that generic neural transitions are necessarily unstable; rather, it makes their long-horizon behavior difficult to regulate and diagnose. Because Dreamer-style actor--critic updates depend on recursively imagined trajectories, errors in latent transition, reward prediction, and continuation prediction can compound with horizon and bias the resulting return targets.

Dynamics in continuous-control tasks are not completely unstructured. Legged locomotion, pendulum control, robotic reaching, and velocity-control tasks often contain smooth action responses, damping, near-periodic motion, and locally linear evolution \cite{tassa2018dmc,todorov2012mujoco}. If such long-horizon patterns can be explicitly represented in latent space, imagined trajectories produced by the world model may become more stable and easier to diagnose. Koopman operator theory provides a useful perspective: a nonlinear system can be advanced by a linear operator in an appropriate observable-function space \cite{koopman1931,koopman1932,mezic2005,brunton2022modern}. Its structural advantage for world models is that the learned latent transition can be organized around an explicit evolution operator, so that contraction, persistence, and oscillatory behavior are controlled by spectral quantities, providing a direct mechanism to analyze and regulate multi-step error growth. Although a finite-dimensional neural representation can only approximate the Koopman property, a Koopman-inspired linear backbone therefore provides a natural way to inject stability-oriented and diagnostically interpretable structure into long-horizon latent dynamics \cite{williams2015edmd,lusch2018deep,takeishi2017koopman}.

This paper proposes Koopman Dreamer. Its main idea is to make long-horizon deterministic-state propagation an explicitly structured component of the world model, instead of treating it only as a generic function-approximation problem. The complete deterministic latent state is organized as a Koopman-inspired evolution state, and its autonomous dynamics are parameterized by a spectral-radius-controlled linear backbone. The remaining components are arranged around this backbone. A linear action term and a bilinear state--action interaction term adapt the backbone to controlled nonlinear dynamics, stochastic-state modulation carries local posterior correction information into the deterministic transition, and posterior-conditioned teacher and rollout objectives reduce the mismatch between posterior training states and prior imagination states. In this way, Koopman Dreamer retains the Dreamer-style framework for latent imagination and actor--critic learning while concentrating its structural contribution on spectrally constrained latent dynamics.

We evaluate Koopman Dreamer on proprioceptive tasks from the DeepMind Control Suite \cite{tassa2018dmc,todorov2012mujoco} and on simulated UAV navigation \cite{krishnan2021airlearning} with vectorized LiDAR observations \cite{miera2023lidar}. Open-loop evaluations show the most consistent improvements in predicting proprioceptive observations and velocity. Reward prediction also improves on most tasks. In UAV navigation, Koopman Dreamer achieves a higher success rate and a lower overall failure rate than DreamerV3. A spectral-radius sweep and structural ablations provide additional exploratory evidence regarding the roles of modal contraction, exponential moving average (EMA) teacher targets, and state-dependent action effects.

The main contributions of this paper are threefold:

\begin{enumerate}[(i)]
\item
We propose Koopman Dreamer for Dreamer-style continuous control. The deterministic latent state is propagated directly through a Koopman-inspired spectral backbone composed of two-dimensional rotation--scaling blocks with bounded modal radii. This parameterization explicitly represents the damping, persistence, and oscillation of latent modes through learnable modal radii and phases.

\item
We adapt the spectral Koopman backbone to controlled nonlinear dynamics and posterior-free imagination. The transition combines a linear action term, a low-rank bilinear state--action interaction term, and stochastic-state modulation. On top of the standard DreamerV3 prior--posterior learning framework, we design Koopman-specific supervision using posterior-conditioned teacher representations, EMA target projections, one-step Koopman consistency, multi-step teacher and prior rollouts, and open-loop observation prediction. Together, these components train the same structured transition under both observation-corrected training and observation-free imagination.

\item
We establish a long-horizon error-propagation analysis for the controlled Koopman transition. The resulting multi-step bound separates autonomous operator amplification from controlled-interaction effects and additive stochastic and modeling residuals, explaining how the learned spectral radius balances error accumulation against the preservation of persistent dynamics. Experiments on DMC proprioceptive control and UAV-LiDAR navigation, including open-loop prediction, spectral-radius sensitivity analysis, and structural ablations, connect this mechanism to improved long-horizon prediction and closed-loop control.
\end{enumerate}

\section{Related Work}\label{ii.-related-work}

\subsection{Latent World Models}\label{a.-latent-world-models}

Latent world models encode observations or observation histories into compact hidden states and predict the future in latent space. Early model-based reinforcement-learning methods often relied on dynamics models defined either in the original observation/state coordinates or in manually designed compact state representations \cite{chua2018pets,janner2019mbpo}. Directly modeling observations can be high-dimensional and redundant, whereas hand-crafted state representations depend on task-specific engineering. These limitations motivate latent world models that learn predictive hidden states from data. Representative methods such as PlaNet and Dreamer combine observation encoding, latent-state transition, observation reconstruction, reward prediction, and policy optimization through stochastic state-space models, so that the hidden state can both explain observations and support future trajectory imagination \cite{hafner2019planet,hafner2020dreamer,hafner2021dreamerv2,hafner2025dreamerv3}. Although the observation front end depends on the representation provided by the environment, these models share the need to propagate the resulting latent state accurately over imagined trajectories.

Modern latent world models differ in architecture, but they share the goal of learning compact dynamics that can support prediction, planning, or policy improvement. In PlaNet and the original Dreamer line, this role is played by recurrent deterministic memory coupled with stochastic state correction; in other model-based agents, latent dynamics are adapted to planning or actor--critic optimization \cite{hafner2019planet,hafner2020dreamer,hafner2021dreamerv2,hafner2025dreamerv3,hansen2024tdmpc2}. Koopman Dreamer is also a learned latent dynamics model, but it constrains deterministic propagation through a spectral parameterization. This makes quantities such as contraction, rotation, and near-periodic modes part of the transition design rather than properties that must emerge only from data and loss terms. This paper therefore focuses on adding explicit modal structure to the deterministic dynamics core. It formulates deterministic latent-state propagation as a spectrally constrained, Koopman-inspired backbone and organizes observation prediction, reward prediction, episode-continuation prediction, and latent-space policy optimization around this structured rollout mechanism.

\subsection{Koopman Representation Learning}
\label{b.-koopman-representation-learning}

Koopman operator theory shows that nonlinear systems can be described by a linear operator in an appropriate observable-function space \cite{koopman1931,koopman1932,mezic2005,brunton2022modern}. Dynamic mode decomposition and extended dynamic mode decomposition estimate linear evolution operators from data, but they typically rely on manually designed observables \cite{rowley2009spectral,schmid2010dmd,williams2015edmd}. Deep Koopman methods use neural networks to learn observable representations and train them jointly with reconstruction, linear-prediction, and multi-step prediction objectives \cite{takeishi2017koopman,lusch2018deep,kaiser2021eigenfunctions}. These methods have been widely studied in system identification, prediction, and control \cite{korda2018koopmanmpc,narasingam2020deepkoopman,shi2022deepkoopmancontrol}.

Koopman representations have also been integrated directly with reinforcement learning and task-oriented control. Koopman Q-learning learns Koopman latent representations to identify dynamical symmetries for offline data augmentation \cite{weissenbacher2022koopmanq}. Task-oriented Koopman control jointly learns a latent embedding, a Koopman operator, and a differentiable linear controller through end-to-end reinforcement learning, including applications with high-dimensional visual and LiDAR observations \cite{lyu2023taskorientedkoopman}. RoboKoop learns control-conditioned visual Koopman representations for off-policy robotic control \cite{kumawat2025robokoop}. Recent control-oriented studies have further combined Koopman models with model-based reinforcement learning for data-driven optimal control of unknown nonlinear systems \cite{zeng2025koopmancontrol}, receding-horizon direct policy optimization for fixed-wing UAV trajectory tracking \cite{li2026rhdpo}, and receding-horizon actor--critic learning for autonomous-vehicle motion planning \cite{zhang2026lpc}.

Long-horizon prediction errors are closely related to the spectral properties of the dynamics operator. An unconstrained linear operator may amplify perturbations exponentially during recursive rollouts, whereas excessive contraction may erase persistent information \cite{arjovsky2016unitary,miller2019stable}. Spectral constraints regulate these behaviors through bounds on matrix norms, eigenvalue magnitudes, or structured operator parameterizations \cite{brunton2022modern}. These ideas motivate structured evolution operators whose modal amplification, decay, and oscillation can be directly parameterized rather than left implicit.

The preceding Koopman studies demonstrate that learned Koopman representations can support policy learning and control. The present work addresses a different aspect: the explicit spectral organization of the deterministic dynamics core inside a stochastic Dreamer-style latent world model for long-horizon recursive imagination. The deterministic latent state is propagated directly through a Koopman-inspired spectral backbone composed of two-dimensional rotation--scaling blocks, whose bounded radii and learnable phases govern contraction, persistence, and oscillation. Linear and bilinear action terms adapt this spectral backbone to state-dependent controlled dynamics. Posterior-conditioned teacher targets and prior-rollout objectives serve as supporting training mechanisms that allow the same structured transition to operate under observation-corrected representation learning and posterior-free imagination.

\section{Problem Formulation}\label{iii.-problem-formulation}
We consider a partially observable continuous-control problem \cite{sutton2018rl}. At time \(t\), the agent receives an observation \(o_t\), takes an action \(a_t\) according to the history \(h_t=(o_{\le t},a_{<t})\), obtains a reward \(r_t\), and receives an episode-continuation flag \(c_t\in\{0,1\}\). Here, \(c_t=1\) indicates that the episode continues after this step, whereas \(c_t=0\) indicates termination. The objective is to learn a policy \(\pi_\psi(a_t\mid h_t)\) that maximizes the expected discounted return gated by continuation, \(J(\pi_\psi)=\mathbb{E}_{\pi_\psi}[\sum_{t=0}^{\infty}\gamma^t(\prod_{i=0}^{t-1}c_i)r_t]\), where \(\psi\) denotes the policy parameters, \(\gamma\in(0,1)\), and the empty product is defined as 1. This formulation stops return accumulation after episode termination and is consistent with the continuation predictor used for imagined returns in Dreamer-style agents \cite{hafner2025dreamerv3}.

The latent world model compresses the history \(h_t\) into a hidden state and models transition, observation, reward, and episode continuation in latent space. Following the deterministic--stochastic state decomposition used by Dreamer-style RSSM world models \cite{hafner2019planet,hafner2020dreamer,hafner2025dreamerv3}, we write the latent state as \(y_t=(\phi_t,s_t)\), where \(\phi_t\in\mathbb{R}^{D}\) is the deterministic Koopman state and \(s_t\) is the stochastic state. The deterministic component carries historical memory that can be recursively propagated over long horizons, while the stochastic component absorbs current-observation correction and unmodeled local factors. In the implementation, \(s_t\in\{0,1\}^{N\times K}\) consists of \(N\) groups of \(K\)-way one-hot discrete stochastic variables. The categorical distributions are smoothed with unimix to prevent premature collapse of the stochastic state \cite{jang2017gumbel,maddison2017concrete,hafner2021dreamerv2,hafner2025dreamerv3}. In subsequent world-model and actor--critic training, \(y_t\) serves as a latent-space summary of the history \(h_t\), and the policy is therefore optimized in latent space as \(\pi_\psi(a_t\mid y_t)\).

Instead of directly sampling the complete next state \(y_{t+1}\), Koopman Dreamer separates deterministic propagation from the stochastic prior. Given \(y_t=(\phi_t,s_t)\) and action \(a_t\), the Koopman backbone first produces \(\phi_{t+1}=F_K(\phi_t,s_t,a_t)\), and the prior network then predicts \(p_\theta(s_{t+1}\mid \phi_{t+1})\). Thus, the complete latent-state transition can be written as
\begin{equation}
p_\theta(y_{t+1}\mid y_t,a_t)
=
\delta
\left(
\phi_{t+1}-F_K(\phi_t,s_t,a_t)
\right)
p_\theta(s_{t+1}\mid \phi_{t+1}),
\label{eq:latent_transition}
\end{equation}
where \(\theta\) collects the world-model parameters and \(\delta(\cdot)\) denotes the Dirac distribution corresponding to the deterministic map. Equation~\eqref{eq:latent_transition} clarifies that stochasticity in the model arises only from the prior distribution of \(s_{t+1}\), whereas \(\phi_{t+1}\) is produced by the structured Koopman transition. The prediction heads further model observations, rewards, and episode-continuation flags.

During training, the model can use the current observation to construct a posterior state. Let \(e_t=E_\theta(o_t)\) be the observation encoding. The posterior inference distribution of the stochastic state is \(q_\theta(s_t\mid \phi_t,e_t)\), following the prior--posterior inference structure of Dreamer-style world models \cite{hafner2019planet,hafner2025dreamerv3}.

The posterior network only corrects the stochastic state \(s_t\) and does not directly resample \(\phi_t\). The deterministic state \(\phi_t\) is still propagated from the previous posterior state and previous action through the Koopman transition introduced below. The trajectory formed by posterior stochastic states is denoted as \(y_t^{\mathrm{post}}\), and the prior rollout obtained along an action sequence without future observations is denoted as \(\tilde y_t\).

The central objective of this paper is to construct a world model that uses a spectrally controlled deterministic Koopman backbone for long-horizon latent-state propagation. Observation encoding, posterior correction, reward prediction, episode-continuation prediction, and policy imagination are therefore organized around the same controllable latent evolution rather than around the unconstrained deterministic update of the baseline RSSM.
\section{Method}\label{iv.-method}
This section presents the architecture and optimization of Koopman Dreamer. We first describe the overall model, observation interface, and latent-state inference. We then introduce the controlled Koopman transition and its spectral parameterization, followed by the posterior-conditioned teacher representation, world-model objectives, and latent actor--critic optimization.

\subsection{Model Overview}\label{a.-model-overview}
Koopman Dreamer augments a DreamerV3 world model \cite{hafner2025dreamerv3} with a spectrally constrained Koopman transition in place of the deterministic RSSM update, together with posterior-conditioned and posterior-free rollout objectives. The model comprises an observation encoder, the Koopman dynamics, stochastic-state prior and posterior networks, observation/reward/continuation heads, and latent-space actor and value networks. The complete latent state is \(y_t=(\phi_t,s_t)\), where \(\phi_t\) is the deterministic Koopman evolution state and \(s_t\) supplies observation-conditioned stochastic correction. Fig.~\ref{fig:main_framework} summarizes the two complementary learning workflows built around the shared latent dynamics.
\begin{figure*}[!t]
\centering
\subfloat{\includegraphics[width=0.49\textwidth]{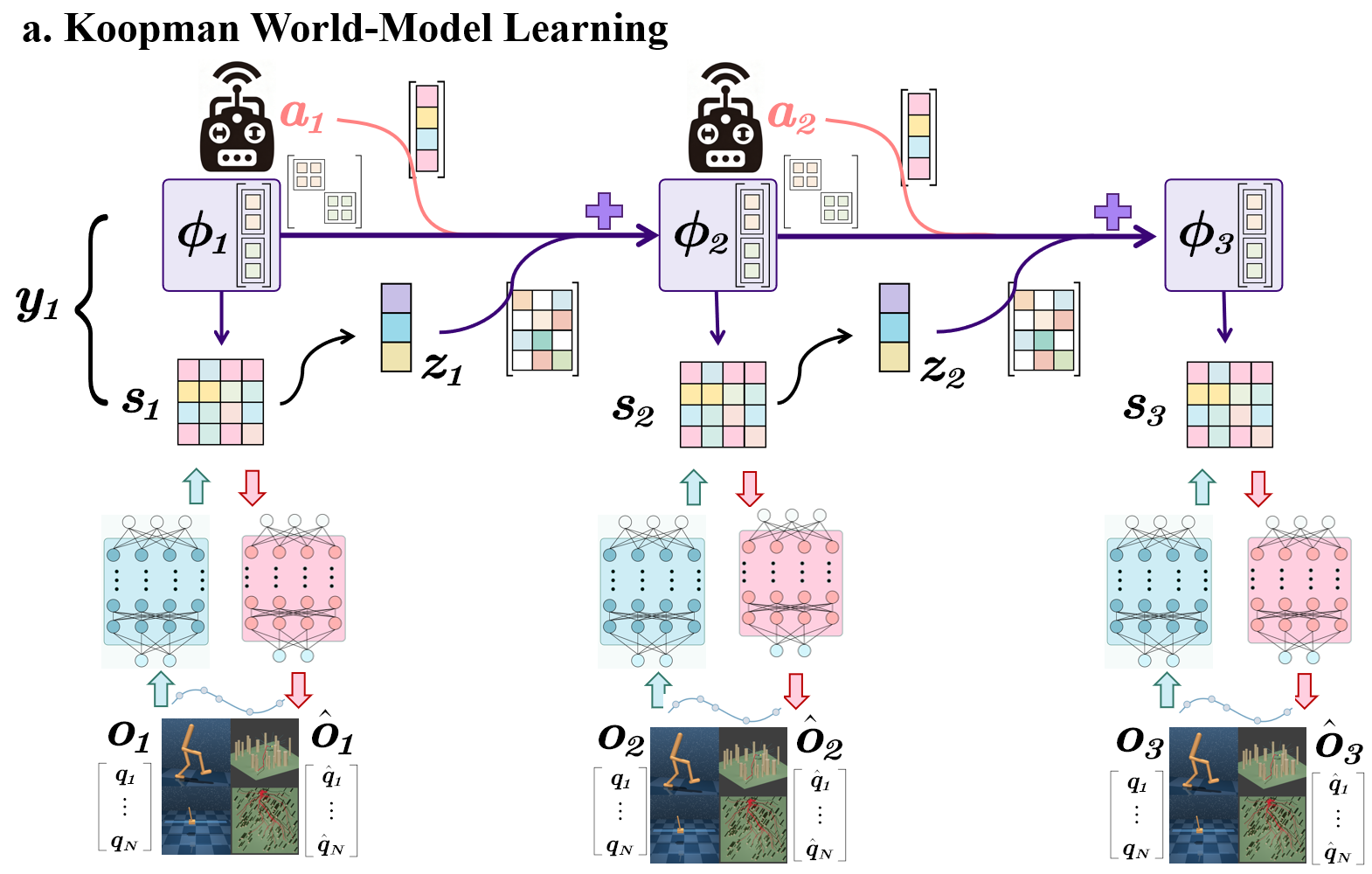}}
\hfill
\subfloat{\includegraphics[width=0.49\textwidth]{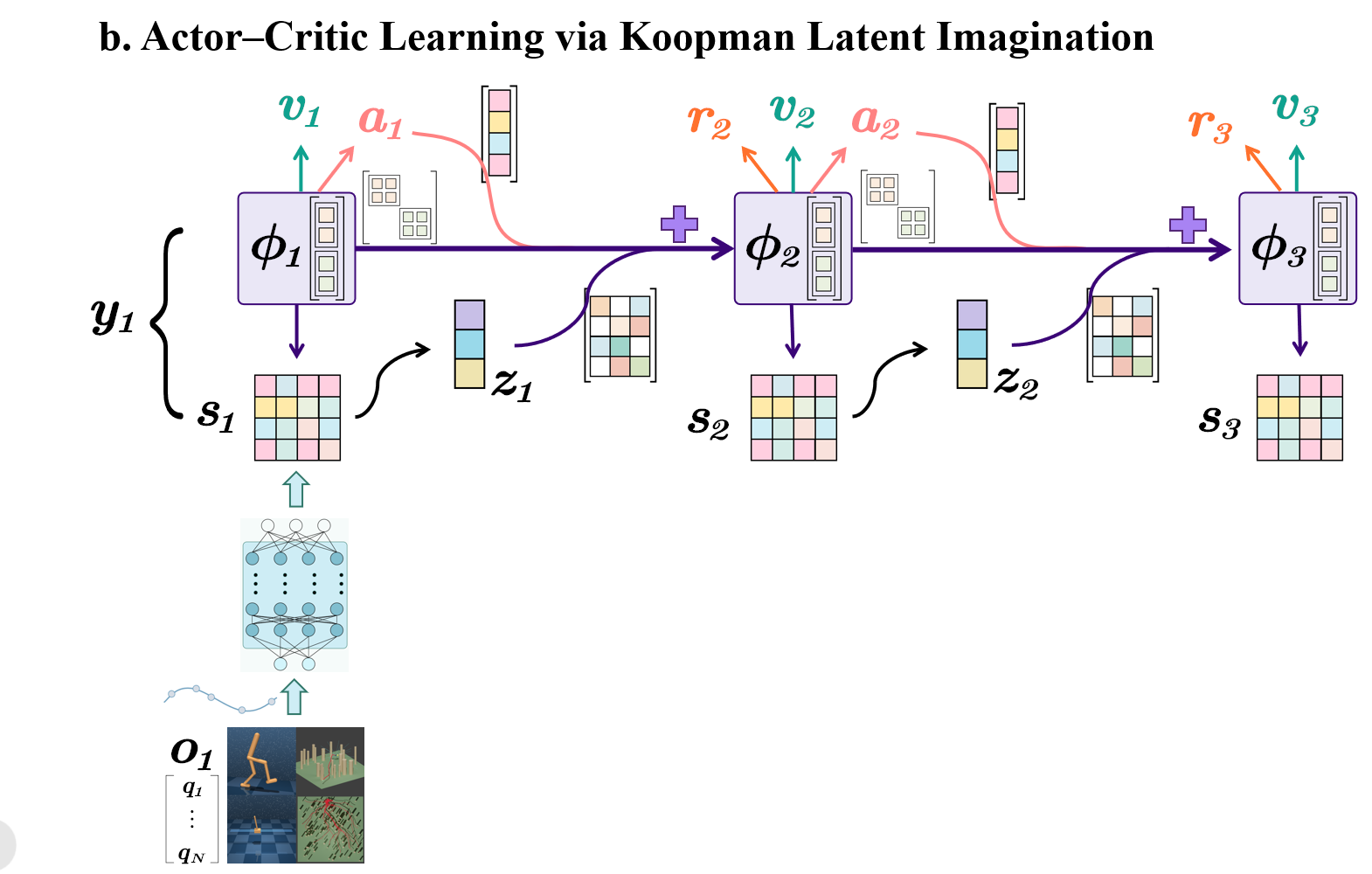}}
\caption{Learning workflows of Koopman Dreamer. (a) Posterior-conditioned world-model learning from observed sequences, where observations provide stochastic-state corrections and prediction targets. (b) Actor--critic learning on prior imagined trajectories initialized from an observation-conditioned state, without access to future observations. Both workflows share the controlled Koopman latent transition detailed in Fig.~\ref{fig:koopman_transition_structure}.}
\label{fig:main_framework}
\end{figure*}
Training and imagination share the same latent dynamics but differ in available information. During the sequence training illustrated in Fig.~\ref{fig:main_framework}(a), the transition first propagates
\(\phi_t=F_K(\phi_{t-1},s_{t-1},a_{t-1})\).
The encoded observation \(e_t=E_\theta(o_t)\) then produces
\(s_t\sim q_\theta(s_t\mid\phi_t,e_t)\), yielding the posterior state
\(y_t^{\mathrm{post}}=(\phi_t,s_t)\) used by the prediction heads and teacher branch. During the open-loop prediction and policy imagination illustrated in Fig.~\ref{fig:main_framework}(b), future observations are unavailable, so stochastic states are obtained from the prior \(p_\theta(s_t\mid\phi_t)\). The auxiliary objectives introduced below train the same structured transition under both information patterns.

\subsection{Observation Encoding and Prediction Heads}\label{b.-observation-encoding-and-prediction-heads}
The observation encoder applies a symlog transform to continuous vector inputs and maps them through an MLP as \(e_t=E_\theta(o_t)\) \cite{hafner2025dreamerv3}. Its input is a proprioceptive state vector for DMC and a task-state--LiDAR vector for UAV navigation. From \(y_t=(\phi_t,s_t)\), the prediction heads model
\(p_\theta(o_t\mid y_t)=\mathrm{Dist}_{o}(D_\theta(y_t))\),
\(p_\theta(r_t\mid y_t)=\mathrm{Dist}_{r}(R_\theta(y_t))\), and
\(p_\theta(c_t\mid y_t)=\mathrm{Bernoulli}(C_\theta(y_t))\).
Continuous observations use symlog reconstruction, rewards use a symexp two-hot distribution, and continuation flags use binary logistic prediction. These heads are supervised on posterior states. On prior imagined states, the reward and continuation heads provide \(\hat r_t\) and \(\hat c_t\) for return estimation.
\subsection{Koopman Latent Transition}\label{c.-koopman-latent-transition}
As shown in Fig.~\ref{fig:koopman_transition_structure}, the Koopman backbone directly propagates the deterministic latent state \(\phi_t\) through a spectrally structured autonomous update augmented by action-dependent and stochastic corrections. The stochastic state \(s_t\) is flattened and linearly projected into a modulation vector \(z_t=g_z(s_t)\), which carries local posterior-correction information into the next deterministic state. The action is first vectorized, and each component is rescaled as
\[
\bar a_{t,j}
=
\frac{[\operatorname{vec}(a_t)]_j}
{\operatorname{sg}\!\left(
\max\!\left(1,\left|[\operatorname{vec}(a_t)]_j\right|\right)
\right)},
\]
where \(j\in\{1,\ldots,d_a\}\) denotes an action-component index, \(d_a=\dim(\operatorname{vec}(a_t))\) is the dimension of the vectorized action, and \(\operatorname{sg}\) denotes stop-gradient. This component-wise rescaling leaves the bounded actions used in our experiments unchanged and only guards against action values outside their nominal range. The Koopman transition is written as
\begin{equation}
\begin{split}
F_K(\phi_t,s_t,a_t)
={}&
\operatorname{clip}_{c_\phi}
\left(
A_K\phi_t+B_a\bar a_t
\right.\\[-1mm]
&\left.\qquad
+H_\theta(\phi_t,\bar a_t)+B_z z_t
\right),
\end{split}
\label{eq:koopman_transition}
\end{equation}
where \(A_K\) is the autonomous Koopman operator, and \(B_a\) and \(B_z\) are learnable linear maps for the action and stochastic modulation, respectively. The element-wise map \(\operatorname{clip}_{c_\phi}(x)=c_\phi\tanh(x/c_\phi)\), with \(c_\phi>0\), bounds the deterministic-state magnitude. The spectral parameterization controls the pre-saturation amplification and modal persistence of the autonomous update. Clipping prevents numerical divergence but does not by itself preserve information or guarantee accurate long-horizon prediction. When a modulation vector is supplied directly, \(F_K^{(z)}(\phi,z,a)\) denotes the same update with \(z\) replacing \(g_z(s)\); thus, \(F_K(\phi,s,a)=F_K^{(z)}(\phi,g_z(s),a)\). The bilinear state--action interaction term adopts a low-rank form:
\begin{equation}
H_\theta(\phi_t,\bar a_t)
=
\beta_b
W_o^b
\left[
\left(W_\phi^b\phi_t\right)
\odot
\left(W_a^b\bar a_t\right)
\right],
\label{eq:bilinear_control}
\end{equation}
where \(\odot\) denotes element-wise multiplication, \(W_\phi^b\) and \(W_a^b\) project the state and action into the same low-rank interaction space, \(W_o^b\) maps the interaction features back to the Koopman-state dimension, and \(\beta_b\) controls the magnitude of the bilinear term. Equation~\eqref{eq:bilinear_control} allows the effect of an action to vary with the current latent state, instead of being restricted to globally fixed action directions through \(B_a\). In ablation configurations, setting \(\beta_b=0\) reduces the transition to one with only the linear action term.

The transition in \eqref{eq:koopman_transition} can be interpreted through four additive components: \(u_t^{K}=A_K\phi_t\) for autonomous spectral latent evolution, \(u_t^{a}=B_a\bar a_t\) for the linear action term, \(u_t^{b}=H_\theta(\phi_t,\bar a_t)\) for state-dependent action modulation, and \(u_t^{z}=B_z z_t\) for local correction information carried by the stochastic state. Thus, \(\phi_{t+1}=\operatorname{clip}_{c_\phi}(u_t^{K}+u_t^{a}+u_t^{b}+u_t^{z})\). This decomposition makes the modeling hierarchy explicit: the long-horizon propagation path is the spectrally controlled autonomous backbone, while the remaining terms express how actions and posterior uncertainty perturb that backbone in controlled tasks.
\begin{figure*}[!t]
\centering
\includegraphics[width=0.98\textwidth]{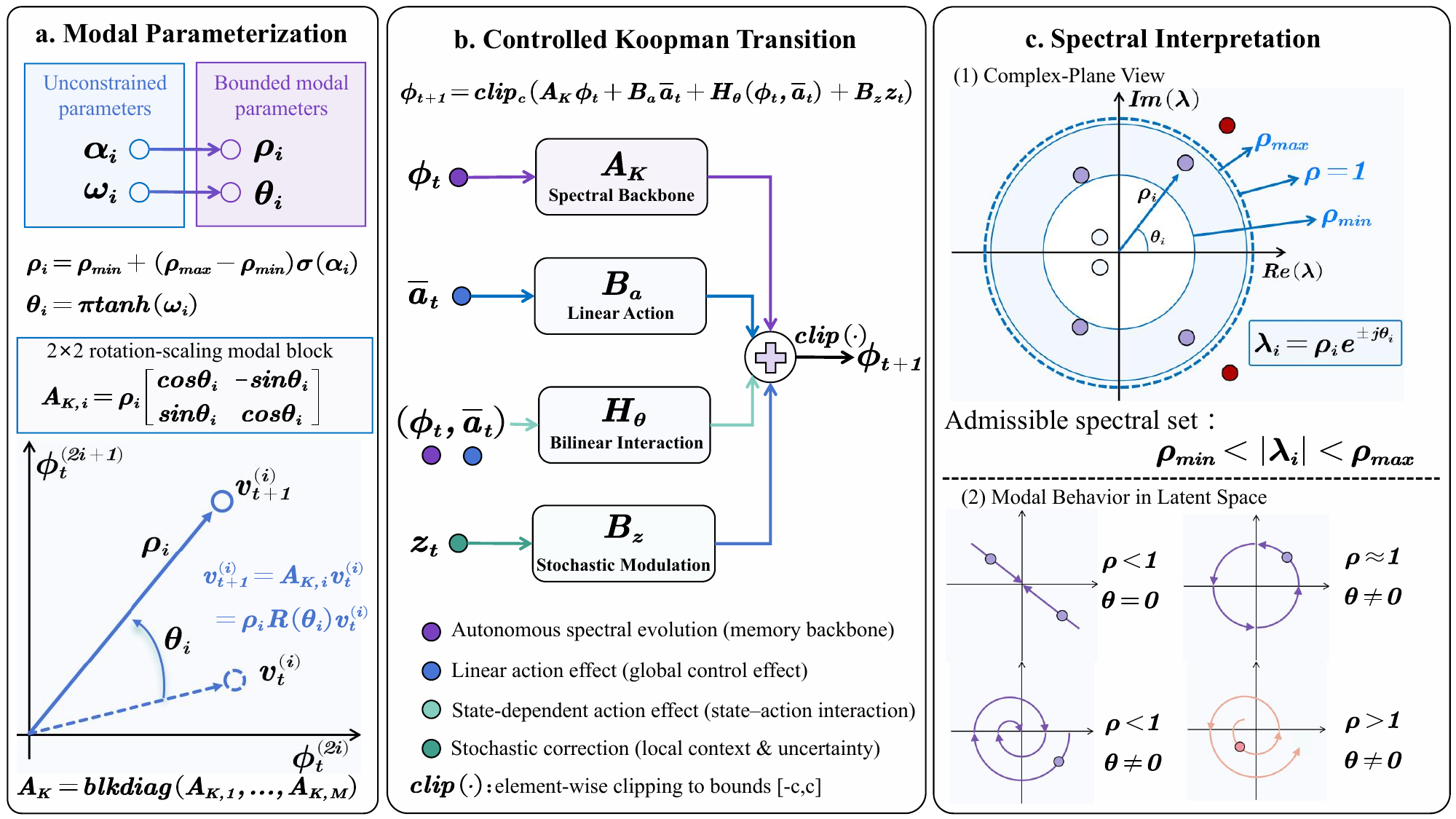}
\caption{Controlled Koopman latent transition and its spectral structure. (a) Bounded modal radii and phases define the two-dimensional rotation--scaling blocks of the autonomous Koopman backbone. (b) The deterministic update combines autonomous spectral evolution, linear and state-dependent action effects, and stochastic-state modulation before smooth element-wise clipping. (c) The eigenvalue geometry and representative latent trajectories illustrate contraction, persistence, rotation, and amplification under different modal radii and phases.}
\label{fig:koopman_transition_structure}
\end{figure*}
\subsection{Spectrally Constrained Koopman Backbone}
\label{d.-spectrally-constrained-koopman-backbone}

As illustrated in Fig.~\ref{fig:koopman_transition_structure}(a) and (c), the Koopman backbone \(A_K\) is parameterized by two-dimensional rotation--scaling blocks, which are the main structural constraint imposed on the deterministic latent dynamics. We use an even deterministic dimension \(D=2M\) throughout this work, so the complete deterministic latent state \(\phi_t\in\mathbb{R}^{D}\) is divided into \(M\) two-dimensional subspaces. For \(i=0,\ldots,M-1\), the action of the Koopman backbone on the \(i\)-th subspace is

\begin{equation}
\begin{bmatrix}
(A_K\phi_t)^{(2i)} \\
(A_K\phi_t)^{(2i+1)}
\end{bmatrix}
=
\rho_i
\begin{bmatrix}
\cos\theta_i & -\sin\theta_i\\
\sin\theta_i & \cos\theta_i
\end{bmatrix}
\begin{bmatrix}
\phi_t^{(2i)} \\
\phi_t^{(2i+1)}
\end{bmatrix}.
\label{eq:spectral_block}
\end{equation}
The modal radius is parameterized as
\(\rho_i=\rho_{\min}+(\rho_{\max}-\rho_{\min})\sigma(\alpha_i)\),
while the phase is
\(\theta_i=\pi\tanh(\omega_{K,i})\).
Here, \(\sigma(\cdot)\) is the logistic sigmoid, and \(\alpha_i\) and \(\omega_{K,i}\) are unconstrained learnable parameters for the modal radius and phase, respectively.

For finite \(\alpha_i\), \(\rho_i\in(\rho_{\min},\rho_{\max})\subseteq[\rho_{\min},\rho_{\max}]\). Since each rotation matrix is orthogonal and \(A_K\) is a block-diagonal normal operator composed entirely of these rotation--scaling blocks, it satisfies

\begin{equation}
\|A_K\|_2
=
\rho(A_K)
=
\max_{0\le i<M}\rho_i
\le
\rho_{\max}.
\label{eq:spectral_bound}
\end{equation}

Here, \(\rho(A_K)\) denotes the spectral radius of \(A_K\).

In the implementation, \(\rho_{\max}\) is usually set close to 1 to allow long-term memory and near-periodic modes, whereas \(\rho_{\min}\) prevents the autonomous modal radii from becoming excessively contractive. It should be emphasized that the spectral constraint in \eqref{eq:spectral_block}--\eqref{eq:spectral_bound} acts directly only on the autonomous Koopman backbone \(A_K\). It does not constitute a global stability proof for the full controlled system that also includes the action term \(B_a\bar a_t\), the bilinear term \(H_\theta\), and the stochastic modulation term \(B_z z_t\). In this paper, the constraint is used as a structured inductive bias for long rollouts, and its empirical role is evaluated through open-loop prediction error, spectral-radius sensitivity analysis, and structural ablation.

\subsection{Stochastic-State Prior and Posterior}\label{e.-stochastic-state-prior-and-posterior}
The stochastic state provides local uncertainty and observation correction around the deterministic Koopman trajectory. Given \(\phi_t\), the prior network produces logits
\(\ell_t^p=P_\theta(\phi_t)\) and
\(p_\theta(s_t\mid\phi_t)=\mathrm{Dist}(\ell_t^p)\).
The posterior additionally conditions on the current observation encoding:
\(\ell_t^q=Q_\theta(\phi_t,e_t)\) and
\(q_\theta(s_t\mid\phi_t,e_t)=\mathrm{Dist}(\ell_t^q)\).
Both distributions use grouped categorical variables with unimix smoothing to avoid premature collapse \cite{hafner2019planet,hafner2021dreamerv2,hafner2025dreamerv3}. Posterior samples provide observation-corrected states for world-model learning and teacher construction, whereas prior samples or means support posterior-free rollout and policy imagination.
\subsection{Posterior-Conditioned Teacher Representation}\label{f.-posterior-conditioned-teacher-representation}
Directly supervising the recursively propagated \(\phi_t\) can couple the transition targets to current model errors. Because the posterior stochastic state contains current-observation information, we use it to construct a corrected teacher trajectory:
\begin{equation}
\begin{aligned}
\bar\phi_t
&=
\operatorname{clip}_{c_\phi}
\left(
\phi_t
+
\beta_{\phi}
\tanh
\left(
W_\phi \operatorname{flat}(s_t)
\right)
\right),\\
\bar z_t
&=
W_z\operatorname{flat}(s_t).
\end{aligned}
\label{eq:teacher_representation}
\end{equation}
Here, \(W_\phi\) and \(W_z\) are teacher projection matrices, and \(\beta_\phi\) bounds the posterior offset. The teacher supplies targets for the same Koopman transition rather than defining a separate dynamics model.

To reduce target drift, the projection parameters are updated in an exponential-moving-average branch,
\(\xi_{\mathrm{ema}}\leftarrow\tau_{\mathrm{ema}}\xi+(1-\tau_{\mathrm{ema}})\xi_{\mathrm{ema}}\), where \(\tau_{\mathrm{ema}}=0.01\) is the weight assigned to the online parameters.
The stop-gradient targets are
\(\bar\phi_t^{\mathrm{tar}}=\operatorname{sg}(\bar\phi_t^{\mathrm{ema}})\) and
\(\bar z_t^{\mathrm{tar}}=\operatorname{sg}(\bar z_t^{\mathrm{ema}})\).
Only the two teacher projection heads have EMA copies; the observation encoder, Koopman transition, prior, and posterior retain their online parameters \cite{tarvainen2017mean,grill2020byol}.
\subsection{World-Model Training Objective}\label{g.-world-model-training-objective}
Let \(q_{b,t}=q_\theta(s_{b,t}\mid\phi_{b,t},e_{b,t})\) and
\(p_{b,t}=p_\theta(s_{b,t}\mid\phi_{b,t})\) for a batch of \(N_b\) sequences of length \(T\). For observation output keys \(m\in\mathcal O\), let
\(\mathcal L_o=\sum_{m\in\mathcal O}\mathcal L_m\) denote the total reconstruction loss. The base world-model objective is
\(\mathcal{L}_{\mathrm{base}}=\lambda_o\mathcal{L}_{o}+\lambda_r\mathcal{L}_{r}+\lambda_c\mathcal{L}_{c}+\lambda_{\mathrm{dyn}}\mathcal{L}_{\mathrm{dyn}}+\lambda_{\mathrm{rep}}\mathcal{L}_{\mathrm{rep}}\).
The observation, reward, and continuation terms use the prediction losses described above. The dynamics term averages
\(\max(\mathrm{KL}(\operatorname{sg}(q_{b,t})\Vert p_{b,t}),\lambda_{\mathrm{free}})\),
and the representation term averages
\(\max(\mathrm{KL}(q_{b,t}\Vert\operatorname{sg}(p_{b,t})),\lambda_{\mathrm{free}})\).
These complementary stop-gradient paths train the prior to follow the posterior while keeping the posterior representation predictable \cite{kingma2014vae,hafner2019planet,hafner2025dreamerv3}. The remaining losses specifically supervise the structured transition.

\textbf{One-step Koopman supervision.}
Local transition consistency is enforced on consecutive EMA teacher states:
\begin{equation}
\begin{aligned}
\mathcal{L}_{\mathrm{koop}}
&=
\frac{1}{N_b(T-1)}
\sum_{b=1}^{N_b}
\sum_{t=0}^{T-2}
\Bigl\|
F_K^{(z)}\bigl(
\operatorname{sg}(\bar\phi_{b,t}^{\mathrm{tar}}),
\operatorname{sg}(\bar z_{b,t}^{\mathrm{tar}}),
a_{b,t}
\bigr)\\
&\quad
-
\operatorname{sg}(\bar\phi_{b,t+1}^{\mathrm{tar}})
\Bigr\|_2^2.
\end{aligned}
\label{eq:koopman_loss}
\end{equation}
Here, \(a_{b,t}\) is the action from \(t\) to \(t+1\). Stop-gradient is applied to the teacher input and target, so this path primarily updates the Koopman backbone and its control terms.

\textbf{Multi-step teacher rollout supervision.}
One-step consistency does not directly expose accumulated recursive error. Each replay sequence supplies one rollout start: for a sequence of length \(T\), we set \(t=0\) and use \(H=T-1\). Starting from
\(\hat\phi_t=\operatorname{sg}(\bar\phi_t^{\mathrm{tar}})\), we roll out
\[
\hat\phi_{t+k}
=
F_K^{(z)}
\left(
\hat\phi_{t+k-1},
\operatorname{sg}(\bar z_{t+k-1}^{\mathrm{tar}}),
a_{t+k-1}
\right)
\]
and minimize
\begin{equation}
\mathcal{L}_{\mathrm{roll}}^{\mathrm{tea}}
=
\frac{1}{N_bH}
\sum_{b=1}^{N_b}
\sum_{k=1}^{H}
w_k^{\mathrm{roll}}
\left\|
\hat\phi_{b,t+k}
-
\operatorname{sg}(\bar\phi_{b,t+k}^{\mathrm{tar}})
\right\|_2^2,
\label{eq:teacher_rollout_loss}
\end{equation}
where \(w_k^{\mathrm{roll}}=1+\tanh(\tau_{\mathrm{roll}}k)\) emphasizes longer horizons for \(\tau_{\mathrm{roll}}\ge0\).

\textbf{Posterior-free rollout supervision.}
Teacher modulation remains observation-conditioned and therefore does not reproduce the information pattern used during imagination. We consequently start from a posterior state and roll out
\(\tilde y_{t+k}\sim p_\theta(\tilde y_{t+k}\mid\tilde y_{t+k-1},a_{t+k-1})\)
without future observations, using the same start \(t\) and horizon \(H\) as in the teacher rollout:
\begin{equation}
\mathcal{L}_{\mathrm{roll}}^{\mathrm{free}}
=
\frac{1}{N_bH}
\sum_{b=1}^{N_b}
\sum_{k=1}^{H}
w_k^{\mathrm{roll}}
\left\|
\bar\phi(\tilde y_{b,t+k})
-
\operatorname{sg}(\bar\phi_{b,t+k}^{\mathrm{tar}})
\right\|_2^2.
\label{eq:free_rollout_loss}
\end{equation}
Here, \(\bar\phi(\tilde y)\) maps a prior rollout state through the online teacher projection head. The combined rollout loss is
\(\mathcal L_{\mathrm{roll}}=(\mathcal L_{\mathrm{roll}}^{\mathrm{tea}}+\mathcal L_{\mathrm{roll}}^{\mathrm{free}})/2\).

\textbf{Open-loop observation prediction.}
Latent alignment alone can yield internally consistent but weakly decodable states. From posterior starting times \(t\in\mathcal S\), we therefore decode observations along prior rollouts:
\begin{equation}
\begin{split}
\mathcal{L}_{\mathrm{pred}}
={}&
\frac{1}{N_b|\mathcal{S}|H}
\sum_{b=1}^{N_b}
\sum_{t\in\mathcal{S}}
\sum_{k=1}^{H}
 w_k^{\mathrm{roll}}\\[-1mm]
&\qquad\times
\sum_{m\in\mathcal{O}}
\ell_m
\left(
D_m(\tilde y_{b,t+k}),
o_{b,t+k}^{m}
\right).
\end{split}
\label{eq:openloop_pred_loss}
\end{equation}
The state \(\tilde y_{b,t+k}\) is generated without observations after the starting time. Here, \(\mathcal S\subseteq\{0,\ldots,T-H-1\}\) is the set of valid posterior starting indices selected from each replay sequence. In the implementation, up to four indices are spaced uniformly over the valid range, with \(H\) shortened when necessary so that every target \(t+k\) remains within the sequence. This loss requires prior states to retain information that can be decoded into future observations.

\textbf{Operator regularization.}
The block parameterization already enforces the spectral interval. We therefore use
\(\mathcal{L}_{\mathrm{opreg}}=\lambda_a\|B_a\|_F^2+\lambda_z\|B_z\|_F^2+\lambda_{\phi}\mathbb{E}_{b,t}[\|\bar\phi_{b,t}-\phi_{b,t}\|_2^2]\)
to limit the linear action correction, stochastic modulation, and posterior-conditioned teacher offset without imposing an additional radius penalty, where \(\mathbb{E}_{b,t}\) denotes the uniform average over batch and time.

The complete world-model objective is
\begin{equation}
\begin{aligned}
\mathcal{L}_{\mathrm{wm}}
&=
\mathcal{L}_{\mathrm{base}}
+
\alpha_{\mathrm{koop}}(n)\lambda_{\mathrm{koop}}\mathcal{L}_{\mathrm{koop}}
+
\alpha_{\mathrm{roll}}(n)\lambda_{\mathrm{roll}}\mathcal{L}_{\mathrm{roll}}\\
&\quad
+
\alpha_{\mathrm{pred}}(n)\lambda_{\mathrm{pred}}\mathcal{L}_{\mathrm{pred}}
+
\alpha_{\mathrm{opreg}}(n)\lambda_{\mathrm{opreg}}\mathcal{L}_{\mathrm{opreg}},
\end{aligned}
\label{eq:world_model_objective}
\end{equation}
where \(n\) is the optimization step and each
\(\alpha(n)\in[0,1]\) linearly warms up its corresponding term. The warmups allow the model to establish its latent representation and prediction heads before the propagation-specific supervision reaches full strength.
\subsection{Latent Imagination and Policy Optimization}\label{h.-latent-imagination-and-policy-optimization}
The policy and value function take the full latent state \(y_t\) as input. Starting states are sampled from posterior replay sequences, after which the policy samples
\(a_t\sim\pi_\psi(a_t\mid y_t)\) and the prior transition generates the next latent state without future observations. The reward and continuation heads provide \(\hat r_t\) and \(\hat c_t\) along these trajectories \cite{hafner2020dreamer,hafner2025dreamerv3}. The recursive return is
\begin{equation}
G_t^\lambda
=
\hat r_t
+
\gamma\hat c_t
\left[
(1-\lambda)V_\omega(y_{t+1})
+
\lambda G_{t+1}^\lambda
\right],
\label{eq:lambda_return}
\end{equation}
with terminal bootstrap from the final imagined value. Imagined losses are weighted by
\(w_0^{\mathrm{imag}}=1\) and
\(w_t^{\mathrm{imag}}=\prod_{i=0}^{t-1}\gamma\hat c_i\), so states following a low continuation probability contribute less to learning.

The value head predicts a categorical distribution \(p_\omega^V(\cdot\mid y_t)\) over symexp-spaced bins. It is trained by continuation-weighted two-hot cross-entropy against the return target \(G_t^\lambda\), together with slow-value regularization toward the prediction of a slowly updated value network, weighted by \(\beta_{\mathrm{slow}}\).

For policy optimization, let \(s_R\) be the running 5th-to-95th percentile range of imagined returns, lower-bounded by 1, and define the normalized advantage \(\widetilde A_t=(G_t^\lambda-V_\omega(y_t))/s_R\). The actor uses continuation-weighted policy gradients with entropy regularization coefficient \(\eta\). Return targets, continuation weights, and advantages are stop-gradient quantities, and actor gradients are not propagated into the world model. Because value and policy learning are both evaluated on prior imagined trajectories, their targets directly depend on the long-horizon accuracy of the structured Koopman transition.
\section{Theoretical Analysis}
\label{sec:theory}

We analyze how the spectral backbone affects error propagation in recursive latent rollouts. All vector norms are Euclidean and matrix norms are induced 2-norms. Let
\(\rho_\star=\rho(A_K)\). Because \(A_K\) is block diagonal with normal rotation--scaling blocks,
\begin{equation}
\begin{aligned}
\|A_K\|_2 &= \rho_\star=\max_i\rho_i\le\rho_{\max},\\
\|A_K^H\|_2 &= \rho_\star^H\le\rho_{\max}^H.
\end{aligned}
\label{eq:theory_spectral_norm}
\end{equation}
Thus, \(\rho_\star\) directly controls autonomous modal amplification before clipping: smaller values attenuate perturbations, whereas values near one preserve slowly varying and oscillatory information.

\subsection{Controlled Rollout-Error Bound}

Define the pre-clipping update
\begin{equation}
\mathcal U_K(\phi,z,\bar a)
=A_K\phi+B_a\bar a+H_\theta(\phi,\bar a)+B_z z.
\label{eq:theory_preclip_update}
\end{equation}
The model follows \(\phi_{t+1}=\operatorname{clip}_{c_\phi}(\mathcal U_K(\phi_t,z_t,\bar a_t))\). We compare it with a reference trajectory
\begin{equation}
\begin{aligned}
\phi'_{t+1}
&=\operatorname{clip}_{c_\phi}
  (\mathcal U_K(\phi'_t,z'_t,\bar a_t))+\zeta_t,\\
\|\zeta_t\|&\le\epsilon_t,
\end{aligned}
\label{eq:theory_reference_rollout}
\end{equation}
where \((\phi'_t,z'_t)\) is a comparison trajectory in the same latent space and \(\zeta_t\) is its one-step transition residual. On the bounded latent region visited during training and imagination, assume that the bilinear term is locally Lipschitz:
\begin{equation}
\begin{aligned}
&\|H_\theta(\phi,\bar a)-H_\theta(\phi',\bar a')\|\\
&\quad\le L_\phi^b\|\phi-\phi'\|
       +L_a^b\|\bar a-\bar a'\|.
\end{aligned}
\label{eq:theory_bilinear_lipschitz}
\end{equation}

\begin{proposition}[Controlled rollout-error propagation]
For a shared action sequence, let \(d_t^\phi=\phi_t-\phi'_t\) and define
\begin{equation}
\begin{aligned}
\kappa_\star&=\rho_\star+L_\phi^b,\\
\delta_t&=\|B_z\|_2\|z_t-z'_t\|+\epsilon_t.
\end{aligned}
\label{eq:theory_kappa_delta}
\end{equation}
Then
\begin{equation}
\|d_{t+1}^\phi\|\le\kappa_\star\|d_t^\phi\|+\delta_t,
\label{eq:theory_one_step}
\end{equation}
and, after \(H\) recursive steps,
\begin{equation}
\|d_{t+H}^\phi\|
\le
\kappa_\star^H\|d_t^\phi\|
+\sum_{i=0}^{H-1}\kappa_\star^{H-1-i}\delta_{t+i}.
\label{eq:theory_multistep}
\end{equation}
\end{proposition}

\begin{IEEEproof}
The element-wise map \(\operatorname{clip}_{c_\phi}(x)=c_\phi\tanh(x/c_\phi)\) is 1-Lipschitz. Hence, the difference between clipped updates is bounded by the difference between their inputs. Under shared actions, the linear action terms cancel. Applying \(\|A_K\|_2=\rho_\star\), the Lipschitz bound in \eqref{eq:theory_bilinear_lipschitz}, and \(\|\zeta_t\|\le\epsilon_t\) gives \eqref{eq:theory_one_step}; repeated substitution yields \eqref{eq:theory_multistep}.
\end{IEEEproof}

The bound separates multiplicative propagation, governed by the spectral backbone and state-dependent bilinear interaction through \(\kappa_\star\), from the additive stochastic-modulation mismatch and modeling residual in \(\delta_t\). If the action sequences differ, \(\delta_t\) gains the term
\((\|B_a\|_2+L_a^b)\|\bar a_t-\bar a'_t\|\). The shared-action form in \eqref{eq:theory_multistep} therefore applies directly to our open-loop diagnostics, while the additional term accounts for policy-induced action differences during imagination. This is a local rollout-error result under the stated assumptions, not a global stability guarantee for the controlled system.

\subsection{Implications for Long-Horizon Learning}

If \(\delta_t\le\bar\delta\), then
\begin{equation}
\begin{aligned}
\|d_{t+H}^\phi\|
&\le\kappa_\star^H\|d_t^\phi\|
  +\bar\delta S_H(\kappa_\star),\\
S_H(\kappa)
&=
\begin{cases}
(1-\kappa^H)/(1-\kappa), & \kappa\ne1,\\
H, & \kappa=1.
\end{cases}
\end{aligned}
\label{eq:theory_multistep_delta}
\end{equation}
For \(\kappa_\star<1\), initial errors decay and accumulated residuals remain bounded by \(\bar\delta/(1-\kappa_\star)\); at \(\kappa_\star=1\), residuals grow linearly; and for \(\kappa_\star>1\), both terms can grow geometrically. Hence, reducing \(\rho_\star\) improves attenuation but may erase persistent task information, while increasing it preserves long-term modes at the cost of greater error amplification. This trade-off motivates the spectral-radius study in Section~\ref{c.-spectral-radius-and-long-horizon-prediction-error}.

The training objectives target complementary terms in the bound. The spectral parameterization restricts the autonomous contribution to \(\kappa_\star\). The one-step Koopman loss reduces empirical transition residuals; the teacher and posterior-free rollout losses penalize their recursive accumulation; and the open-loop prediction loss requires prior states to remain decodable. Finally, the dynamics and representation KL terms encourage posterior--prior alignment and thereby target the stochastic contribution \(\|B_z\|_2\|z_t-z'_t\|\). The spectral structure therefore regulates error propagation, while the auxiliary objectives reduce the errors injected at each step.
\section{Experiments and Analysis}\label{v.-experiments-and-analysis}

This section evaluates Koopman Dreamer from five perspectives. First, we compare continuous-control performance and sample efficiency on DMC proprioceptive tasks. Second, we evaluate DMC open-loop prediction to relate control performance to long-horizon model accuracy. Third, we examine closed-loop navigation and open-loop prediction in the UAV-LiDAR Forest scenario. Fourth, we study how the learned spectral radius affects long-horizon prediction. Finally, structural ablations isolate the contributions of the spectral constraint, posterior-conditioned teacher, and bilinear control term. Experimental configurations and complete supplementary results are provided in the appendix.

\subsection{DMC Proprioceptive Continuous Control}
\label{a.-dmc-proprioceptive-continuous-control}

\textbf{Setup.}
We evaluate nine proprioceptive tasks from the DeepMind Control Suite \cite{tassa2018dmc,todorov2012mujoco}, covering swing-up, reaching, balance, and locomotion. The comparison includes PPO, D4PG, DDPG, MPO, DMPO, and DreamerV3 \cite{schulman2017ppo,lillicrap2015ddpg,barthmaron2018d4pg,abdolmaleki2018mpo,abdolmaleki2020distributional,hafner2025dreamerv3}. The baseline curves and scores are taken from the public DMC proprioceptive benchmark reported with DreamerV3, whereas Koopman Dreamer is trained in this work. All methods use low-dimensional state observations, a 500K-step interaction budget, and action repeat 2. Following the public protocol, we average the last recorded score over the available seeds. For Koopman Dreamer, we use \([\rho_{\min},\rho_{\max}]=[0.85,0.95]\). Complete benchmark-source and hyperparameter details are given in Tables~\ref{tab:appendix_acme_config} and \ref{tab:appendix_dmc_config}.

\textbf{Control performance.}
Table~\ref{tab:dmc_last50} shows that Koopman Dreamer achieves the best final score on six tasks and exceeds DreamerV3 on eight. The largest gains occur on Acrobot and Hopper Stand. Fig.~\ref{fig:dmc_learning_curves} further shows continued improvement on Acrobot, an earlier high-return plateau on Hopper Stand, and faster entry into the high-return region on Reacher Hard and Walker Walk. These tasks contain smooth, persistent, or near-periodic dynamics, for which explicitly controlling latent modal persistence is well matched to multi-step imagination.

\begin{table*}[!t]
\caption{Final Scores on Nine DMC Proprioceptive Tasks}
\label{tab:dmc_last50}
\centering
\footnotesize
\setlength{\tabcolsep}{3pt}
\begin{tabular*}{\textwidth}{@{\extracolsep{\fill}}lrrrrrrr@{}}
\toprule
Task & PPO & D4PG & DDPG & MPO & DMPO & DreamerV3 & \shortstack{Koopman\\Dreamer} \\
\midrule
Acrobot & 2.0 & 120.2 & 87.4 & 79.7 & 118.2 & 131.7 & \textbf{292.3} \\
Cartpole & 265.8 & \textbf{876.2} & 863.0 & 857.2 & 859.6 & 856.0 & 873.3 \\
Cheetah & 100.5 & 630.1 & 602.6 & 646.9 & 603.4 & 596.3 & \textbf{692.8} \\
Hopper Stand & 0.0 & 732.3 & 618.9 & 450.3 & 545.3 & 650.4 & \textbf{859.1} \\
Reacher & 662.2 & 973.3 & 964.0 & 967.8 & 977.7 & 961.8 & \textbf{979.0} \\
Reacher Hard & 168.4 & 967.0 & 973.6 & 954.5 & 965.3 & 948.6 & \textbf{987.0} \\
Walker Stand & 151.7 & 917.2 & 974.4 & 959.2 & 980.2 & 957.2 & \textbf{982.4} \\
Walker Walk & 77.0 & \textbf{961.1} & 954.0 & 929.6 & 938.9 & 881.0 & 942.3 \\
Walker Run & 31.4 & 611.9 & 564.3 & 570.5 & 511.5 & \textbf{655.5} & 572.0 \\
\bottomrule
\end{tabular*}
\par\vspace{0.35em}
\parbox{\textwidth}{\scriptsize Each entry is the mean across seeds of the last recorded score; bold indicates the best result in each row.}
\end{table*}

\begin{figure*}[!t]
\centering
\includegraphics[width=0.31\textwidth]{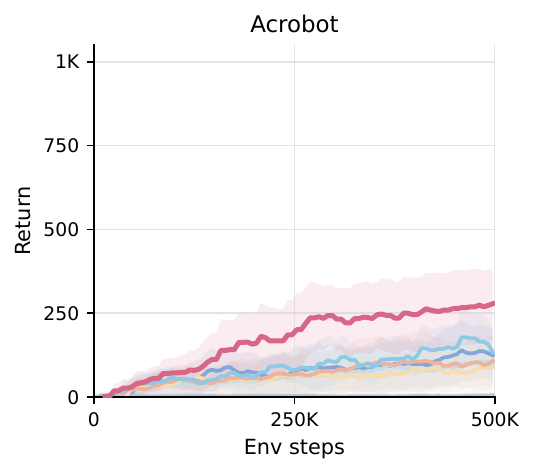}
\hfil
\includegraphics[width=0.31\textwidth]{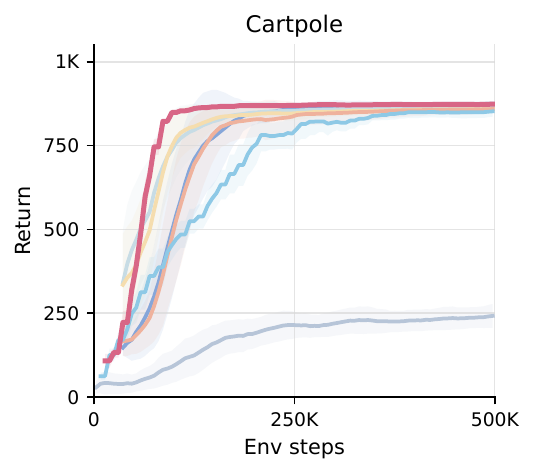}
\hfil
\includegraphics[width=0.31\textwidth]{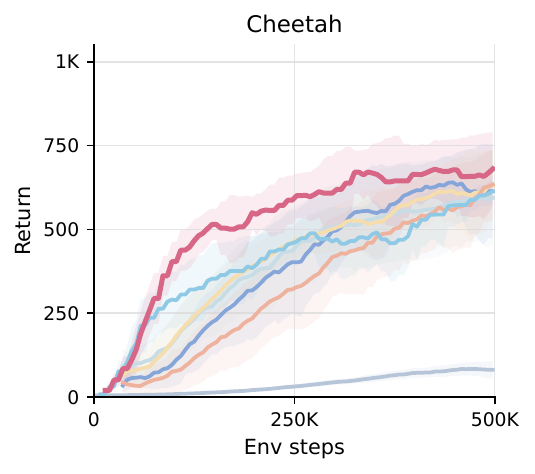}
\par\vspace{0.35em}
\includegraphics[width=0.31\textwidth]{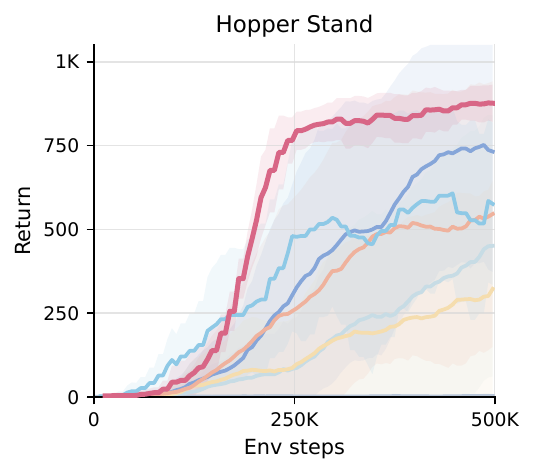}
\hfil
\includegraphics[width=0.31\textwidth]{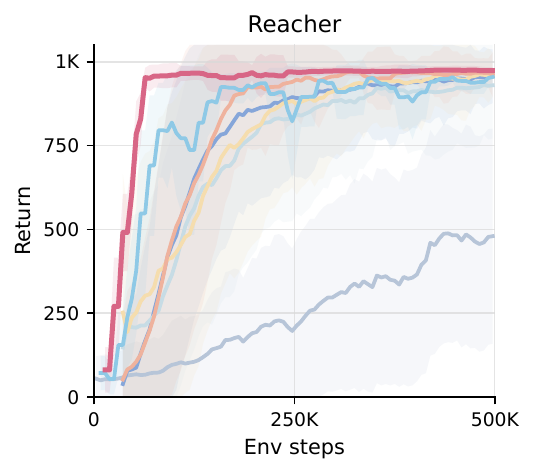}
\hfil
\includegraphics[width=0.31\textwidth]{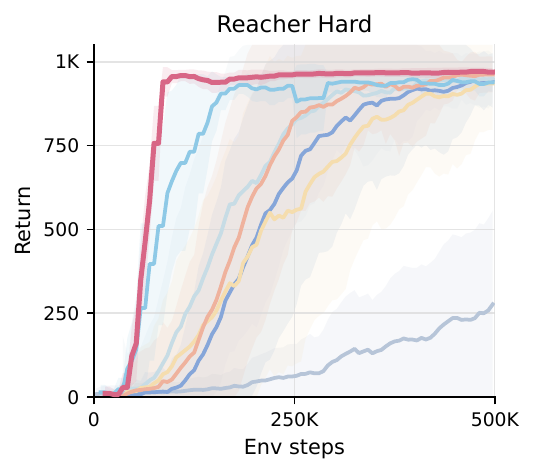}
\par\vspace{0.35em}
\includegraphics[width=0.31\textwidth]{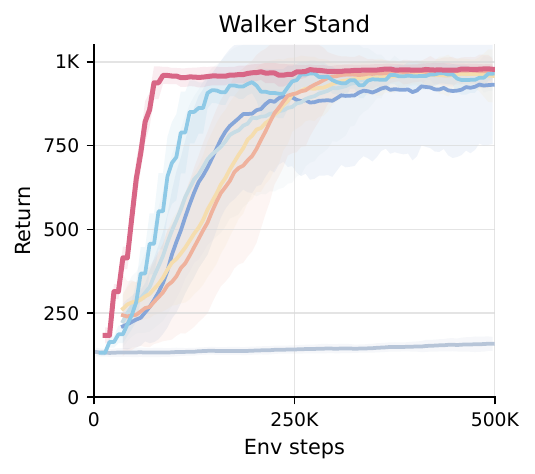}
\hfil
\includegraphics[width=0.31\textwidth]{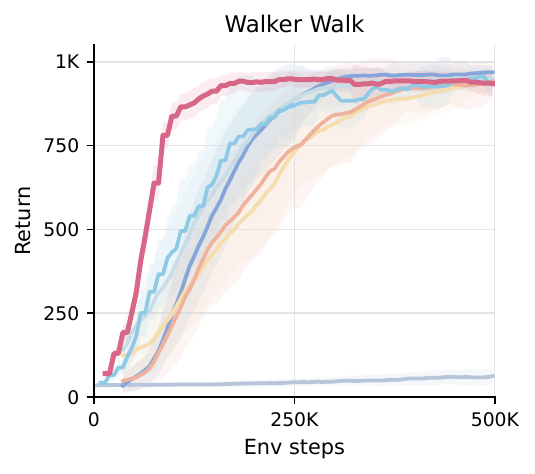}
\hfil
\includegraphics[width=0.31\textwidth]{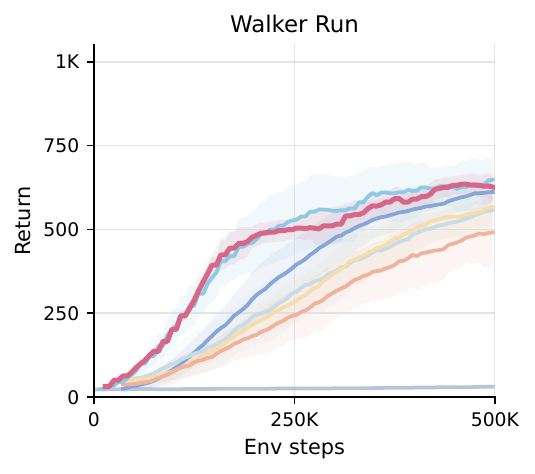}
\par\vspace{0.35em}
\includegraphics[width=0.75\textwidth]{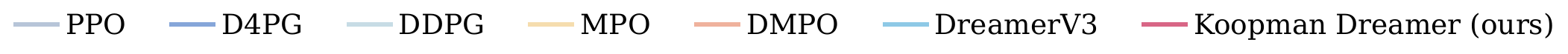}
\caption{Learning curves on nine DMC proprioceptive tasks.}
\label{fig:dmc_learning_curves}
\end{figure*}

Fig.~\ref{fig:dmc_threshold} summarizes sample efficiency using a task-specific target equal to 90\% of the best final score. Koopman Dreamer reaches this target fastest on eight tasks and earlier than DreamerV3 on the same eight. Together, the final scores and learning speeds indicate that the spectral transition improves not only the attainable policy quality but also how rapidly imagined rollouts become useful for policy optimization.

\begin{figure*}[!t]
\centering
\includegraphics[width=0.86\textwidth]{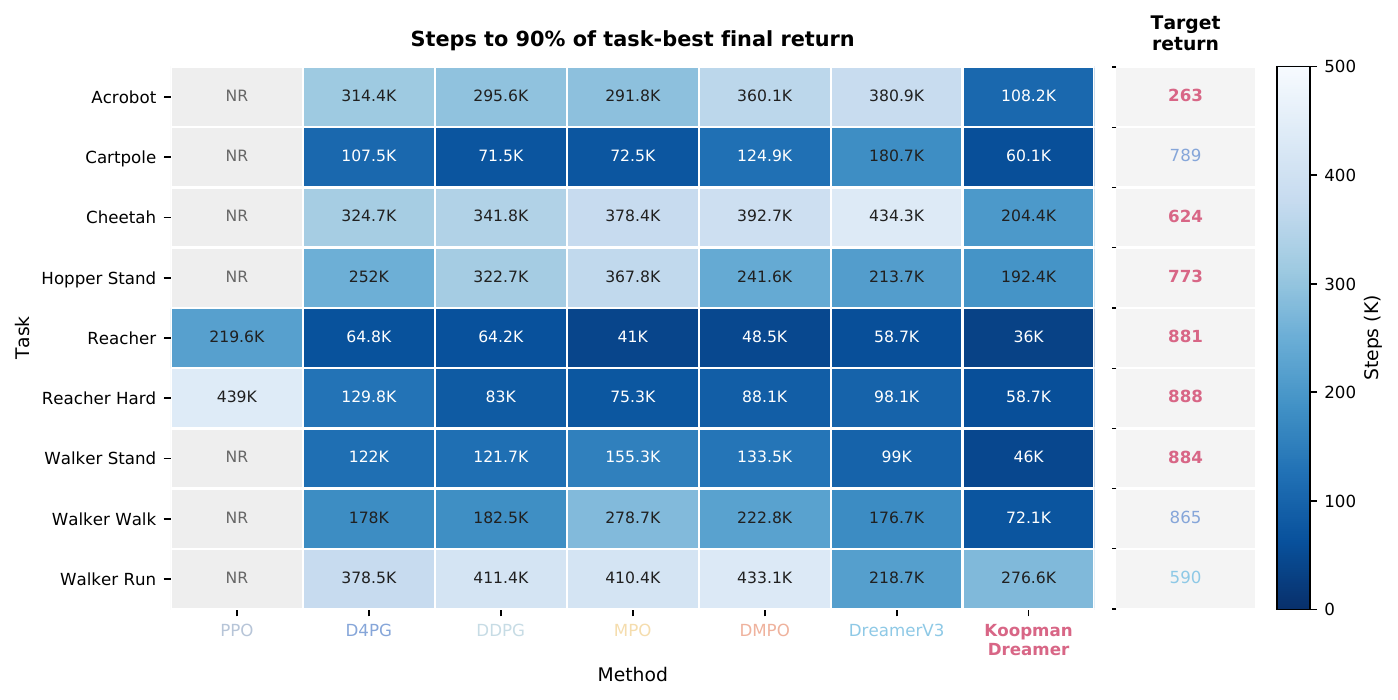}
\caption{Environment steps required to reach 90\% of the task-best final return. NR denotes that the target was not reached within 500K steps.}
\label{fig:dmc_threshold}
\end{figure*}

\textbf{Open-loop prediction.}
Each world model is initialized from a 32-step observation context and rolled out for 64 steps along recorded actions without future observations. Errors are averaged over 32 replay batches with 16 sequences per batch. Table~\ref{tab:dmc_openloop_summary} reports the resulting cross-task errors; complete task-level curves are provided in Figs.~\ref{fig:appendix_dmc_openloop_all_1} and \ref{fig:appendix_dmc_openloop_all_2}.

\begin{table*}[!t]
\caption{Cross-Task Open-Loop Prediction MSE and Comparative Summary on DMC Proprioceptive Tasks}
\label{tab:dmc_openloop_summary}
\centering
\footnotesize
\setlength{\tabcolsep}{6pt}
\renewcommand{\arraystretch}{1.08}
\begin{tabular*}{\textwidth}{@{\extracolsep{\fill}}llcccc@{}}
\toprule
\multicolumn{2}{c}{Evaluation setting} & \multicolumn{2}{c}{Prediction MSE} & \multicolumn{2}{c}{Cross-method comparison} \\
\cmidrule(lr){1-2}\cmidrule(lr){3-4}\cmidrule(lr){5-6}
Evaluation quantity & Rollout summary & DreamerV3 & \shortstack{Koopman\\Dreamer} & \shortstack{Relative MSE\\(Koopman Dreamer/\\DreamerV3)} & \shortstack{Tasks favoring\\Koopman Dreamer} \\
\midrule
Proprioceptive observation & Mean over $H=1{:}64$ & 0.6256 & \textbf{0.4802} & 0.768 & 8/9 \\
& At $H=64$ & 0.7960 & \textbf{0.6953} & 0.873 & 7/9 \\
\addlinespace[2pt]
Reward & Mean over $H=1{:}64$ & 0.9290 & \textbf{0.7318} & 0.788 & 8/9 \\
& At $H=64$ & \textbf{1.1920} & 1.2588 & 1.056 & 8/9 \\
\addlinespace[2pt]
Deterministic latent state & Mean over $H=1{:}64$ & 0.0379 & \textbf{0.00403} & 0.106 & 9/9 \\
& At $H=64$ & 0.0667 & \textbf{0.00698} & 0.105 & 9/9 \\
\bottomrule
\end{tabular*}
\renewcommand{\arraystretch}{1.0}
\par\vspace{0.35em}
\parbox{\textwidth}{\scriptsize Each model is initialized from a 32-step observation context and then rolled out along recorded action sequences without future observations up to horizon 64. The two method columns report MSE averaged across the nine DMC tasks; bold indicates the lower cross-task MSE in each row. ``Mean over \(H=1{:}64\)'' averages prediction error over all 64 open-loop rollout steps. Relative MSE is Koopman Dreamer divided by DreamerV3, so a value below one favors Koopman Dreamer. The final column counts the tasks, out of nine, on which Koopman Dreamer obtains lower MSE.}
\end{table*}

Koopman Dreamer reduces mean deterministic-latent MSE by 89.4\% and has lower horizon-64 latent error on all nine tasks. More importantly, the improvement transfers to shared observation coordinates: mean proprioceptive MSE falls by 23.2\%, with lower error on eight tasks. Reward prediction also improves on eight tasks, and its cross-task MSE averaged over the rollout decreases from 0.9290 to 0.7318. Because separately learned latent coordinates are not directly comparable physical quantities, decoded observation error provides the stronger evidence. The agreement between latent consistency, observable prediction, and closed-loop performance supports the intended role of the spectral backbone: it limits recursive drift while retaining task-relevant dynamics for imagination.

\subsection{UAV-LiDAR Autonomous Navigation}
\label{b.-uav-lidar-navigation-experiments}

\textbf{Environment and protocol.}
The UAV simulator is implemented as a Gymnasium environment \cite{towers2024gymnasium}. The UAV uses three-dimensional velocity commands \(a_t\in[-1,1]^3\), and its observation combines velocity, normalized height, target-relative position, previous action, and LiDAR ranges. Table~\ref{tab:uav_scenes} summarizes the three environments. The main Forest benchmark uses a 50\,m map, 120 cylindrical obstacles, and a 280-dimensional seven-ring LiDAR, coupling long-horizon target progress with dense obstacle sensing. The simulator uses a 0.1\,s control interval, maximum horizontal and vertical speeds of 2.0 and 1.0\,m/s, respectively, and a 500-step episode limit.

\begin{table}[!t]
\caption{UAV-LiDAR Scenario Settings}
\label{tab:uav_scenes}
\centering
\footnotesize
\setlength{\tabcolsep}{2pt}
\begin{tabular}{@{}p{0.16\columnwidth}p{0.43\columnwidth}p{0.33\columnwidth}@{}}
\toprule
Scenario & Setting & Experimental role \\
\midrule
Empty & No obstacle, 20 m, 96-dimensional 3-ring LiDAR & Environment and action-dynamics check \\
Sparse & 20 m, 20 sparse cylindrical obstacles, 96-dimensional 3-ring LiDAR & Sample-efficiency and prediction evaluation \\
Forest & 50 m, 120 dense cylindrical obstacles, 280-dimensional 7-ring LiDAR & Main autonomous navigation evaluation \\
\bottomrule
\end{tabular}
\end{table}

We compare D4PG, DreamerV3, and Koopman Dreamer on four targets. For each method, one fixed checkpoint is evaluated without retraining in three Forest scenes with different obstacle placements shared across methods. Each target uses 20 matched initial conditions per scene, yielding \(240\) closed-loop episodes per method. All methods use a 500K-step training budget. DreamerV3 and Koopman Dreamer use \texttt{size25m} configurations with batch size 16, sequence length 64, imagination horizon 15, and \(\lambda\)-return parameter 0.95. Koopman Dreamer uses a 3072-dimensional deterministic state, 32 groups of 24-way stochastic variables, a 96-dimensional stochastic projection, \([\rho_{\min},\rho_{\max}]=[0.75,0.95]\), and a rank-256 bilinear term; its learned spectral radius is 0.881. Open-loop prediction is evaluated only for DreamerV3 and Koopman Dreamer because D4PG has no world model. Complete environment, evaluation, and method configurations are reported in Tables~\ref{tab:appendix_uav_config} and \ref{tab:appendix_uav_method_config}.

\textbf{Closed-loop navigation.}
The trajectory views in Fig.~\ref{fig:uav_forest_trajectory} show that high-return Koopman Dreamer rollouts more consistently approach the target region with less extended wandering. These selected trajectories are qualitative illustrations rather than the basis of the comparison, but they visualize the role of state-dependent control: the bilinear term allows the same velocity command to have different latent effects across positions, velocities, and obstacle configurations.

\begin{figure*}[!t]
\centering
\begin{minipage}[t]{0.325\textwidth}
\centering
\footnotesize\textbf{\strut D4PG}\par\vspace{0.15em}
\includegraphics[width=0.49\linewidth]{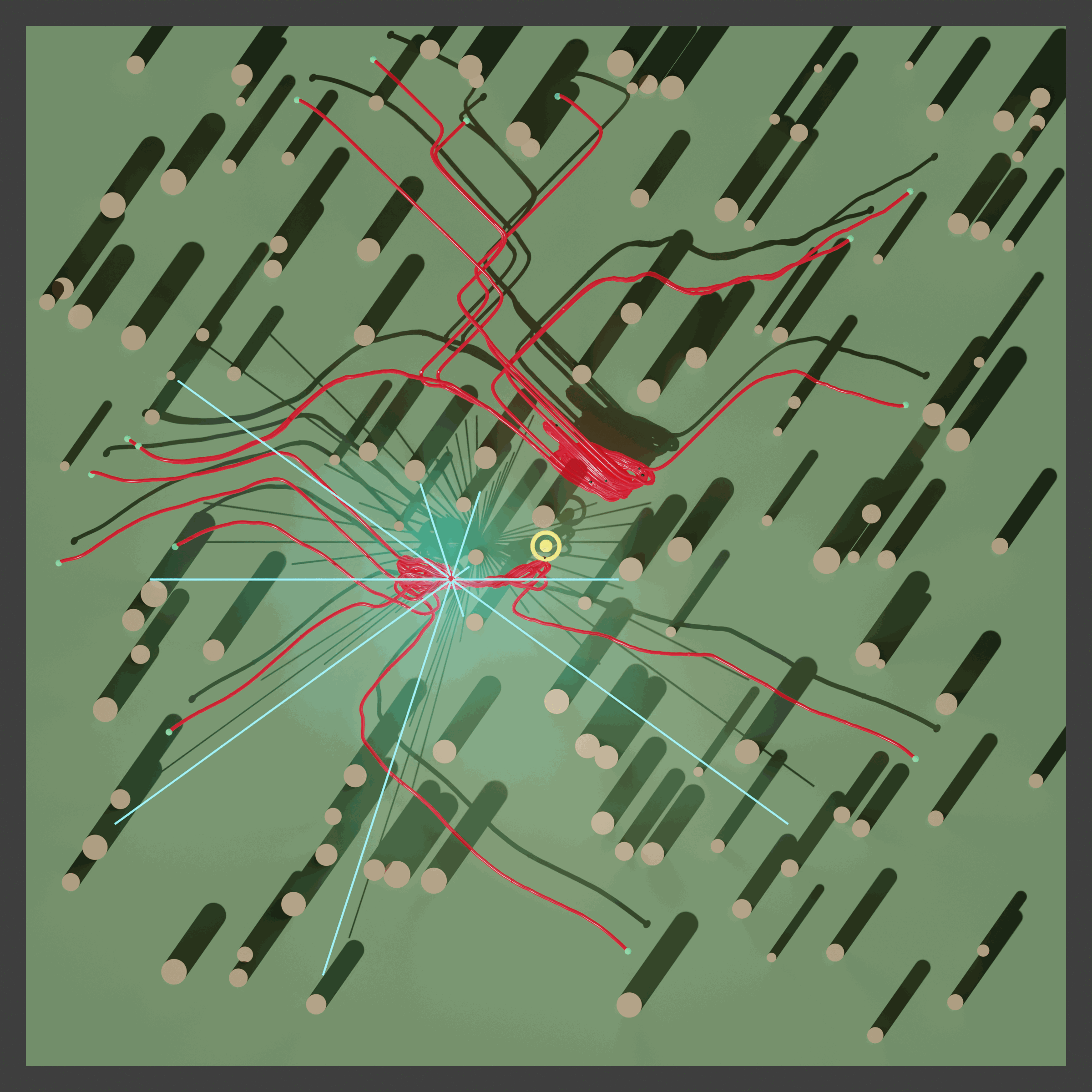}\hfill
\includegraphics[width=0.49\linewidth]{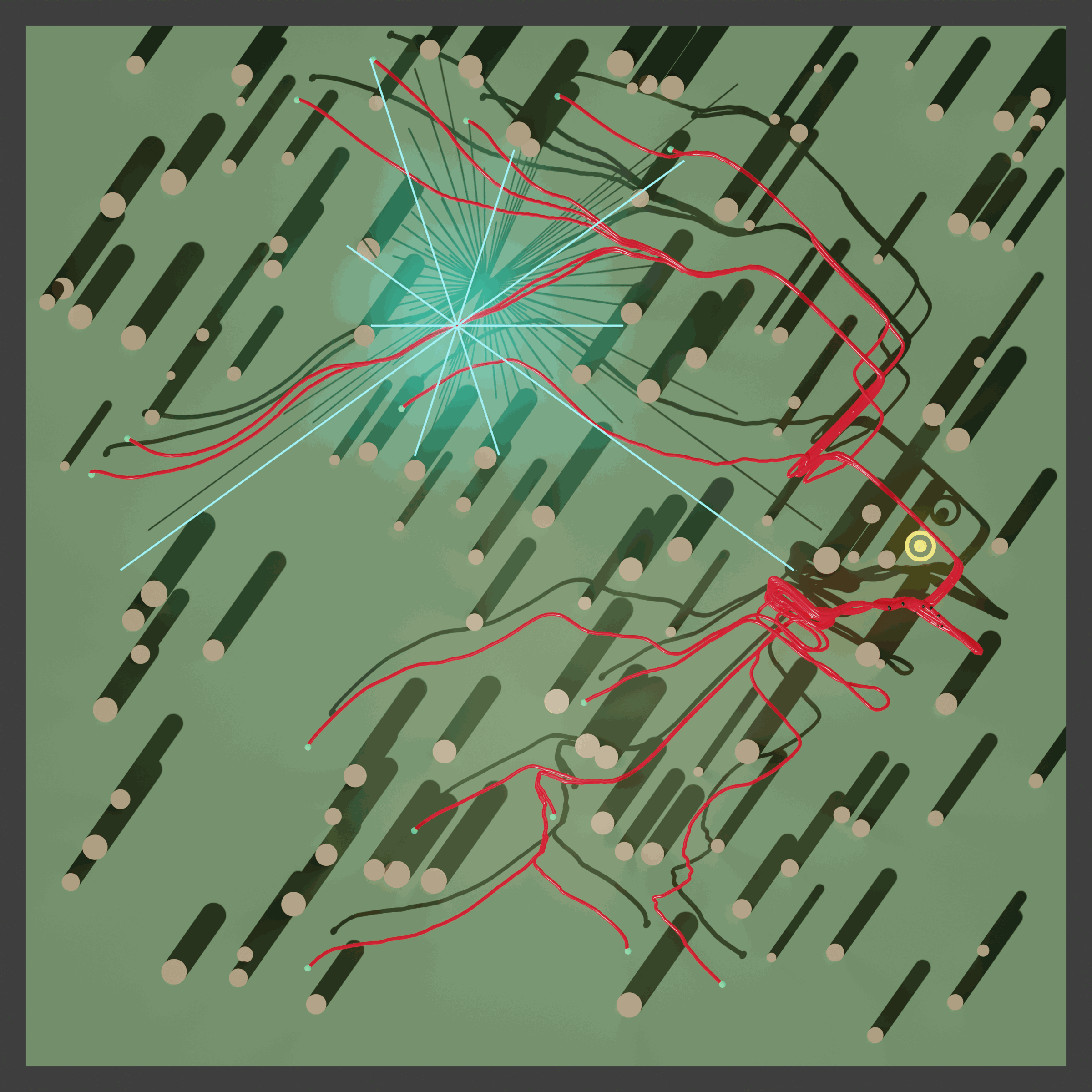}
\par\vspace{-0.45em}
{\scriptsize\makebox[0.49\linewidth][c]{(a)}\hfill\makebox[0.49\linewidth][c]{(b)}}
\par\vspace{0.05em}
\includegraphics[width=0.49\linewidth]{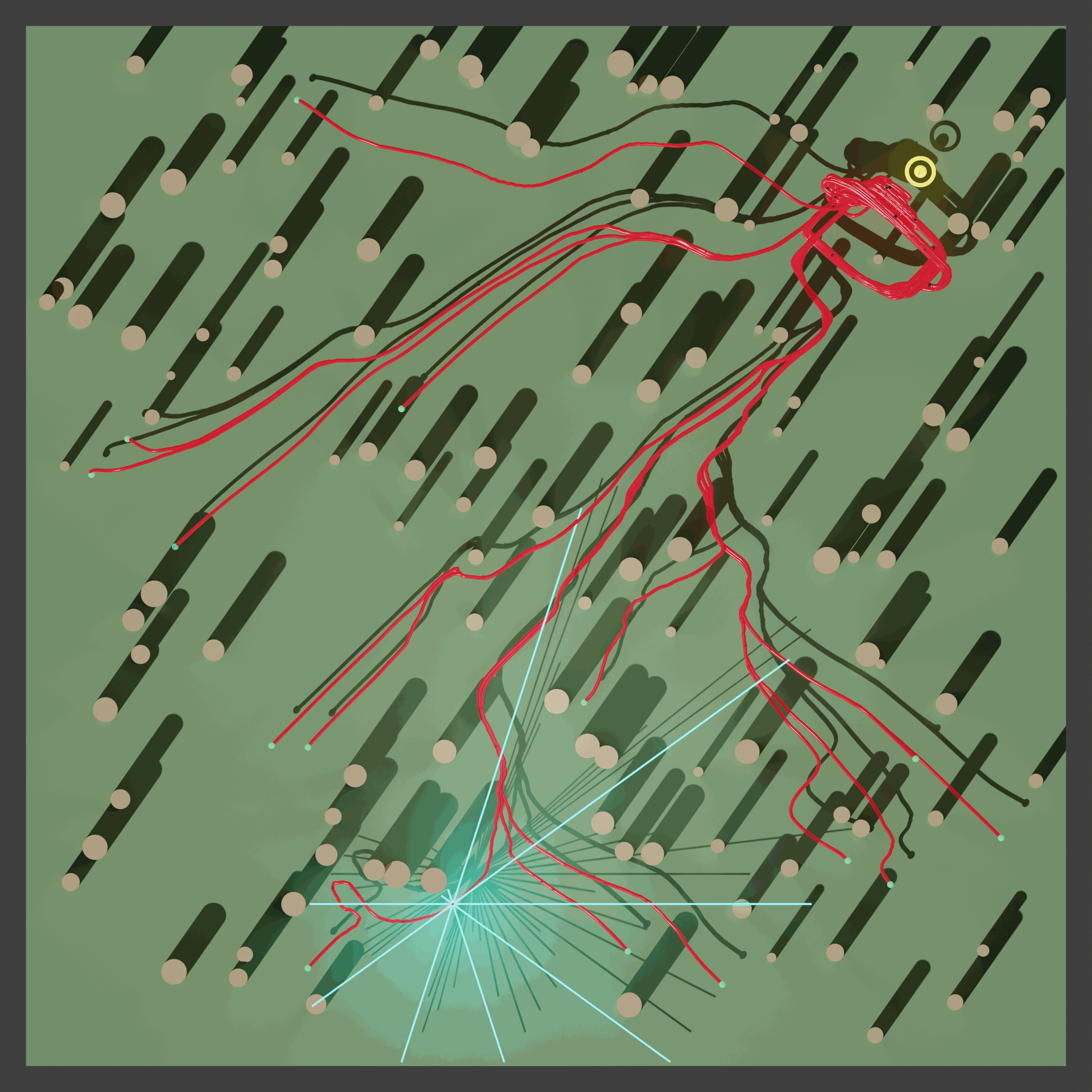}\hfill
\includegraphics[width=0.49\linewidth]{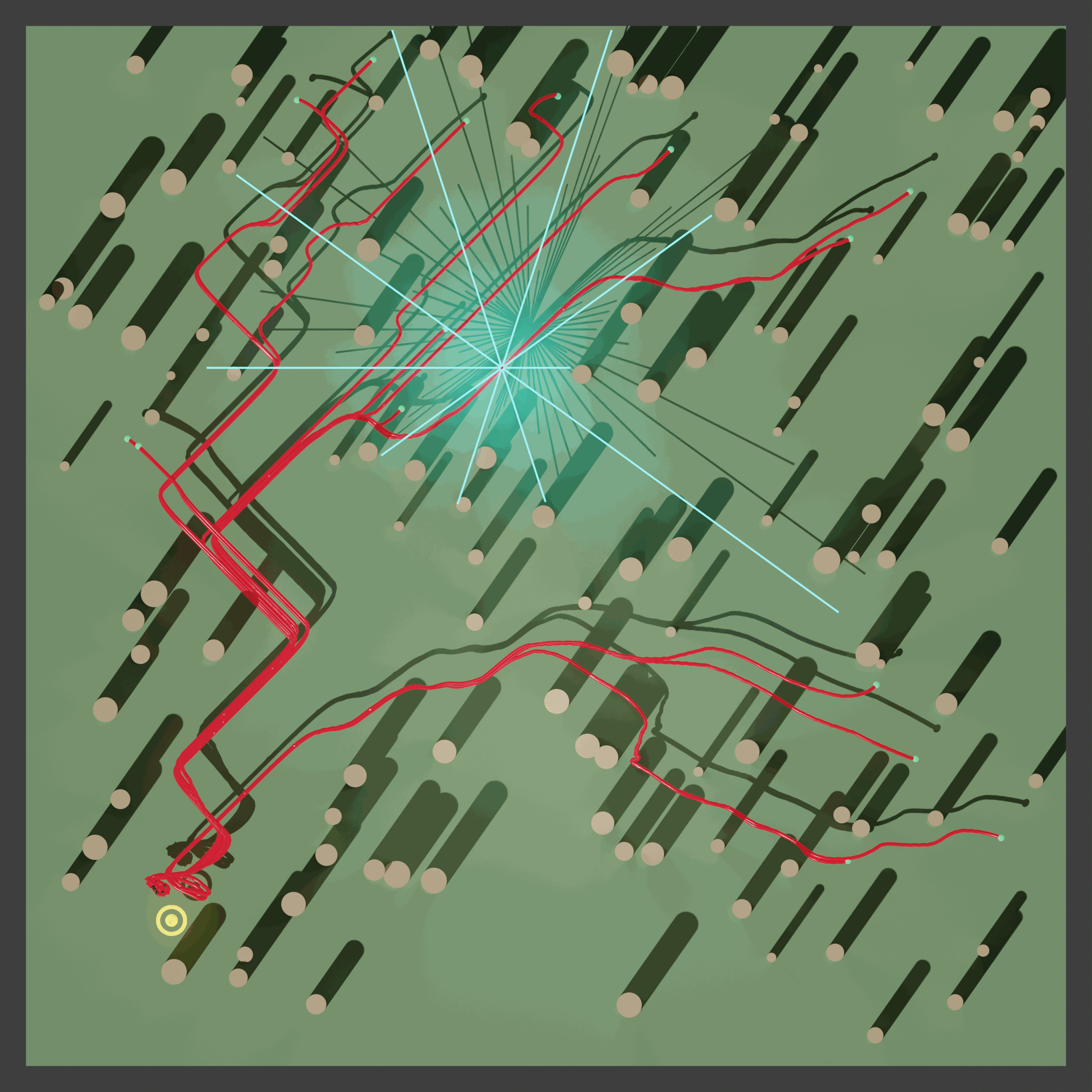}
\par\vspace{-0.45em}
{\scriptsize\makebox[0.49\linewidth][c]{(c)}\hfill\makebox[0.49\linewidth][c]{(d)}}
\end{minipage}\hfill
\begin{minipage}[t]{0.325\textwidth}
\centering
\footnotesize\textbf{\strut DreamerV3}\par\vspace{0.15em}
\includegraphics[width=0.49\linewidth]{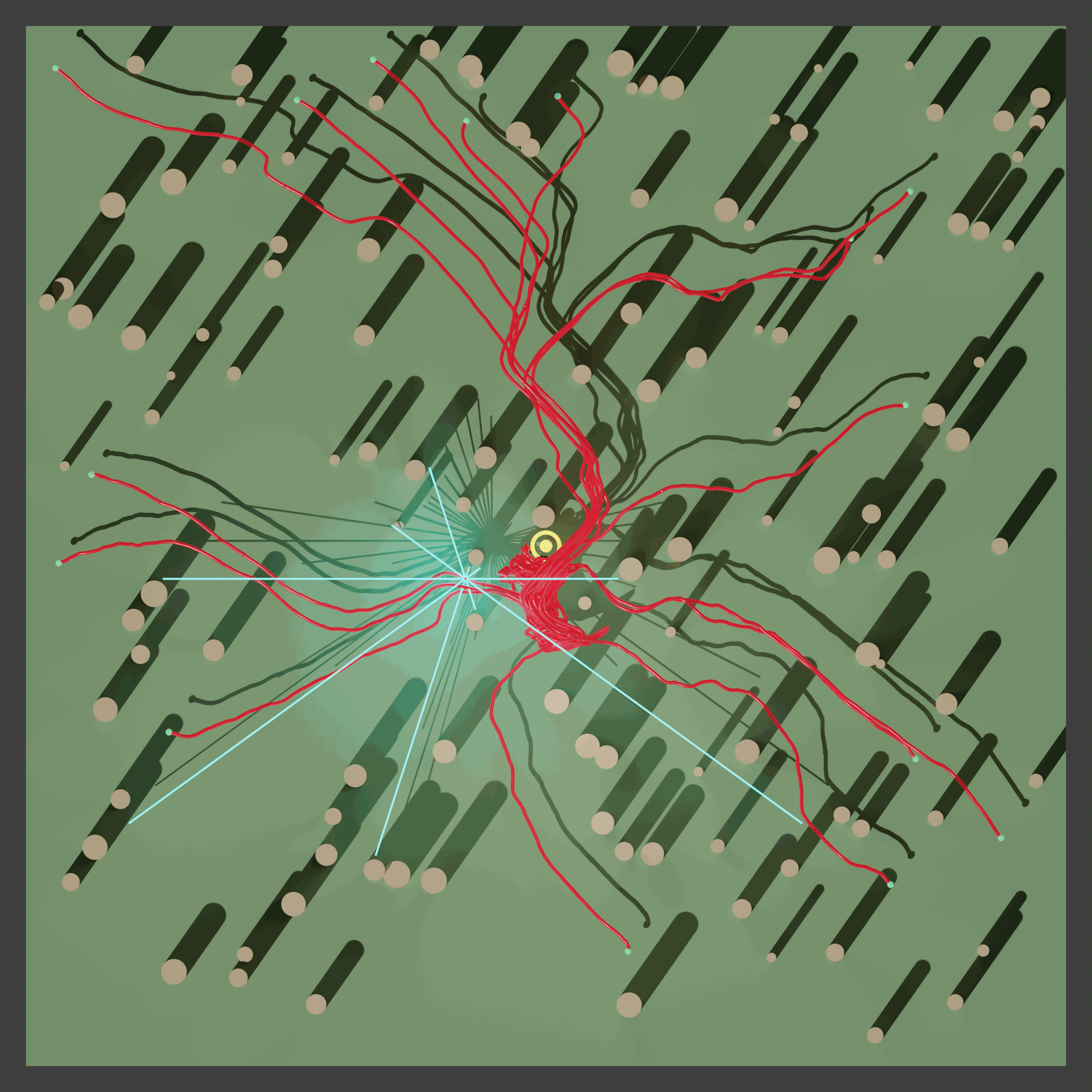}\hfill
\includegraphics[width=0.49\linewidth]{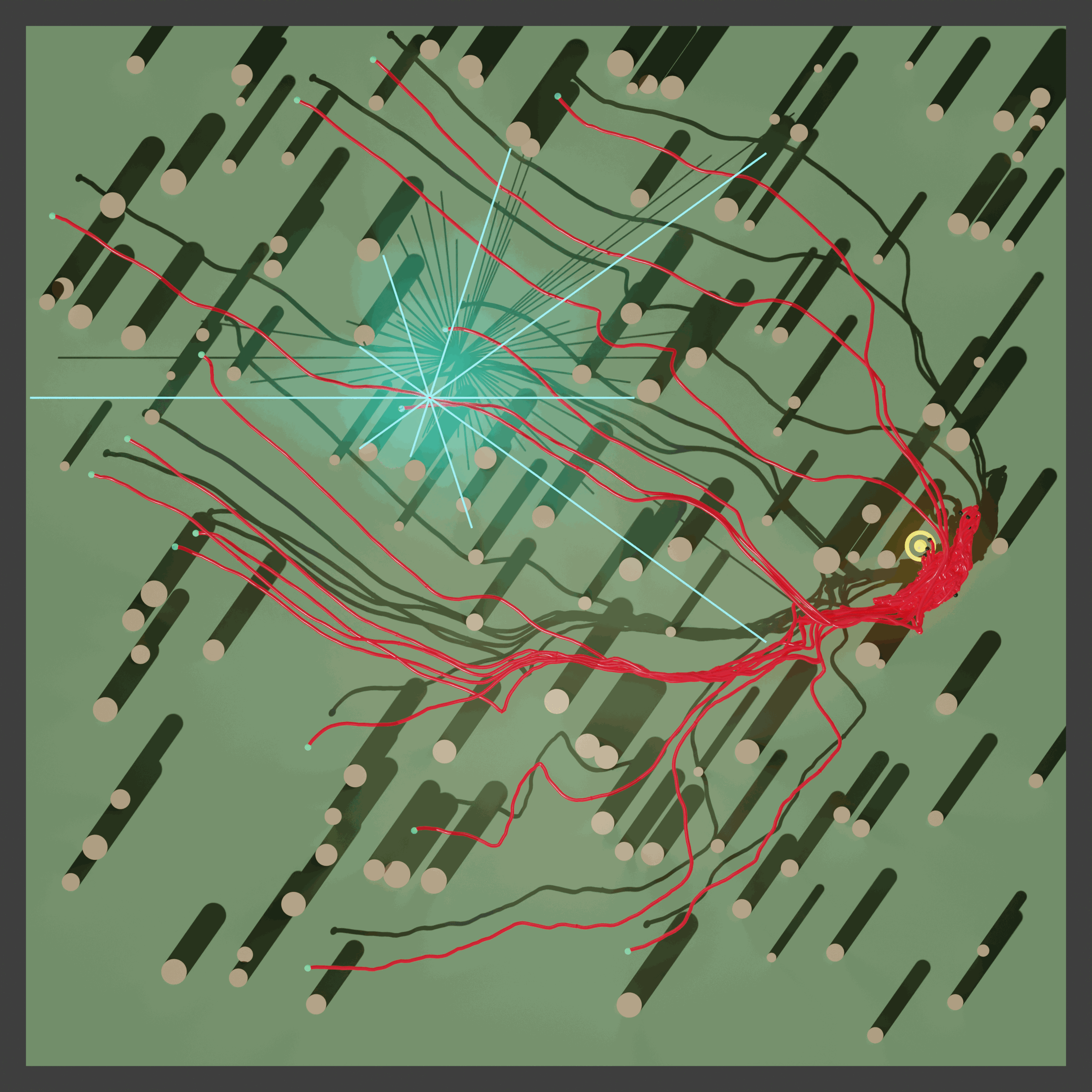}
\par\vspace{-0.45em}
{\scriptsize\makebox[0.49\linewidth][c]{(a)}\hfill\makebox[0.49\linewidth][c]{(b)}}
\par\vspace{0.05em}
\includegraphics[width=0.49\linewidth]{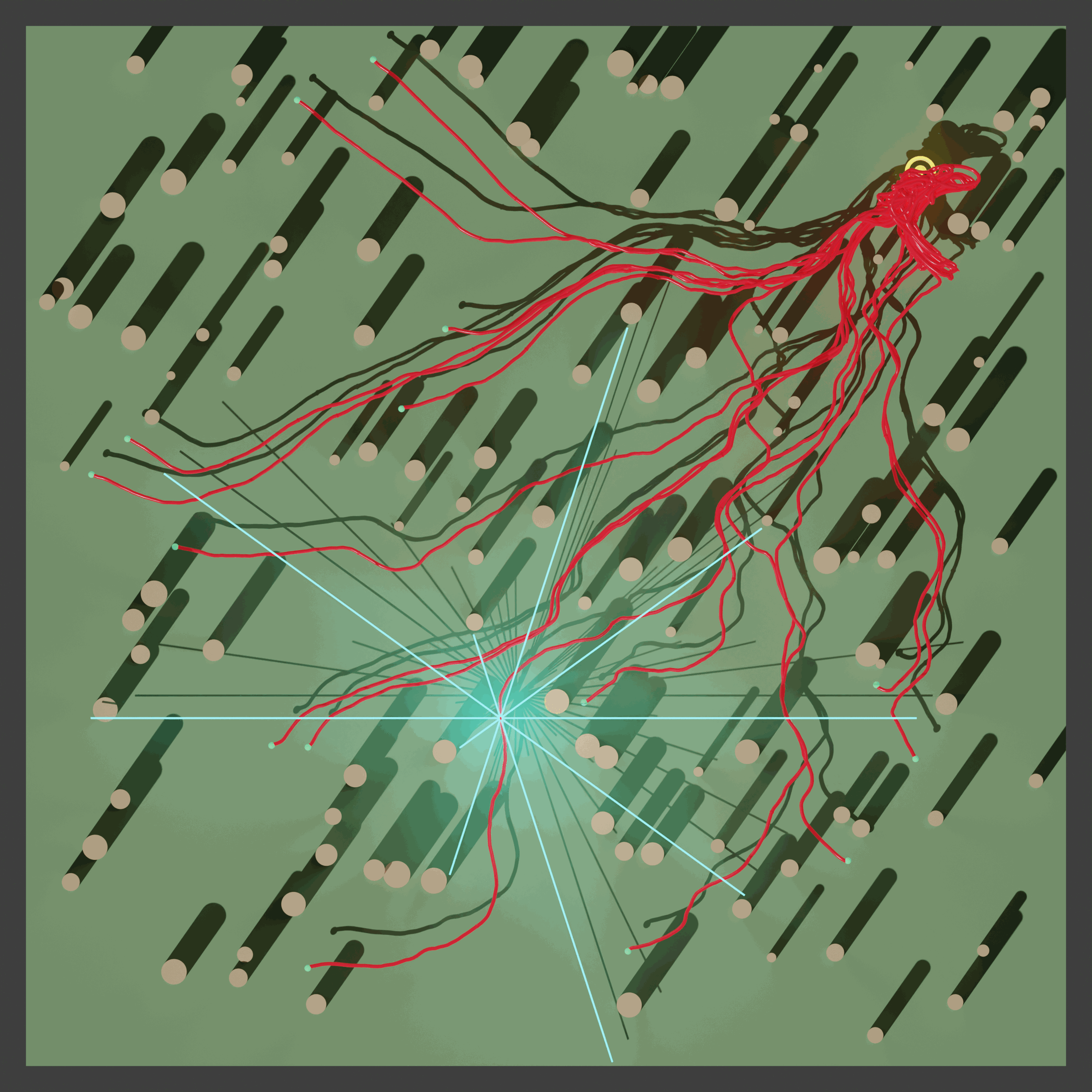}\hfill
\includegraphics[width=0.49\linewidth]{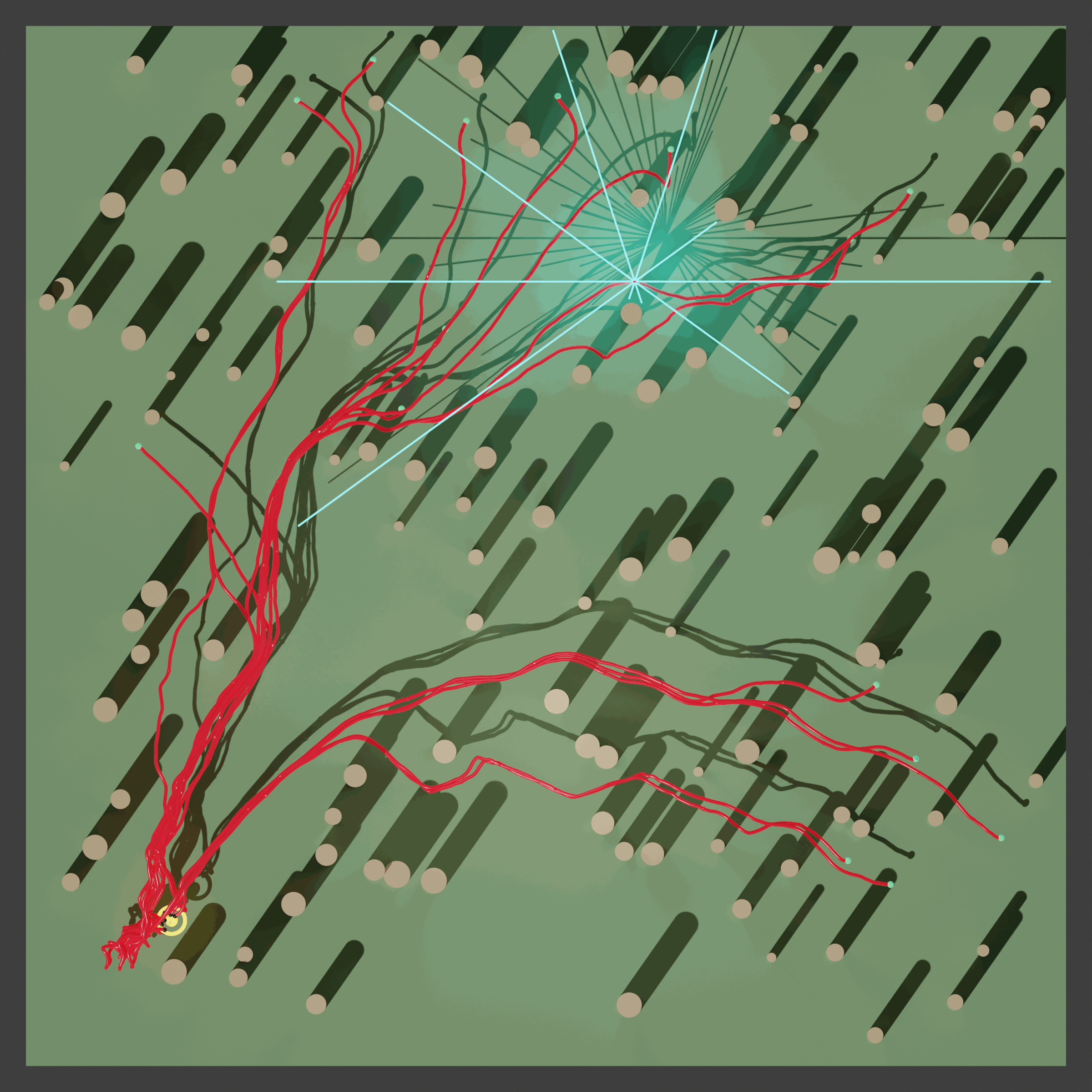}
\par\vspace{-0.45em}
{\scriptsize\makebox[0.49\linewidth][c]{(c)}\hfill\makebox[0.49\linewidth][c]{(d)}}
\end{minipage}\hfill
\begin{minipage}[t]{0.325\textwidth}
\centering
\footnotesize\textbf{\strut Koopman Dreamer}\par\vspace{0.15em}
\includegraphics[width=0.49\linewidth]{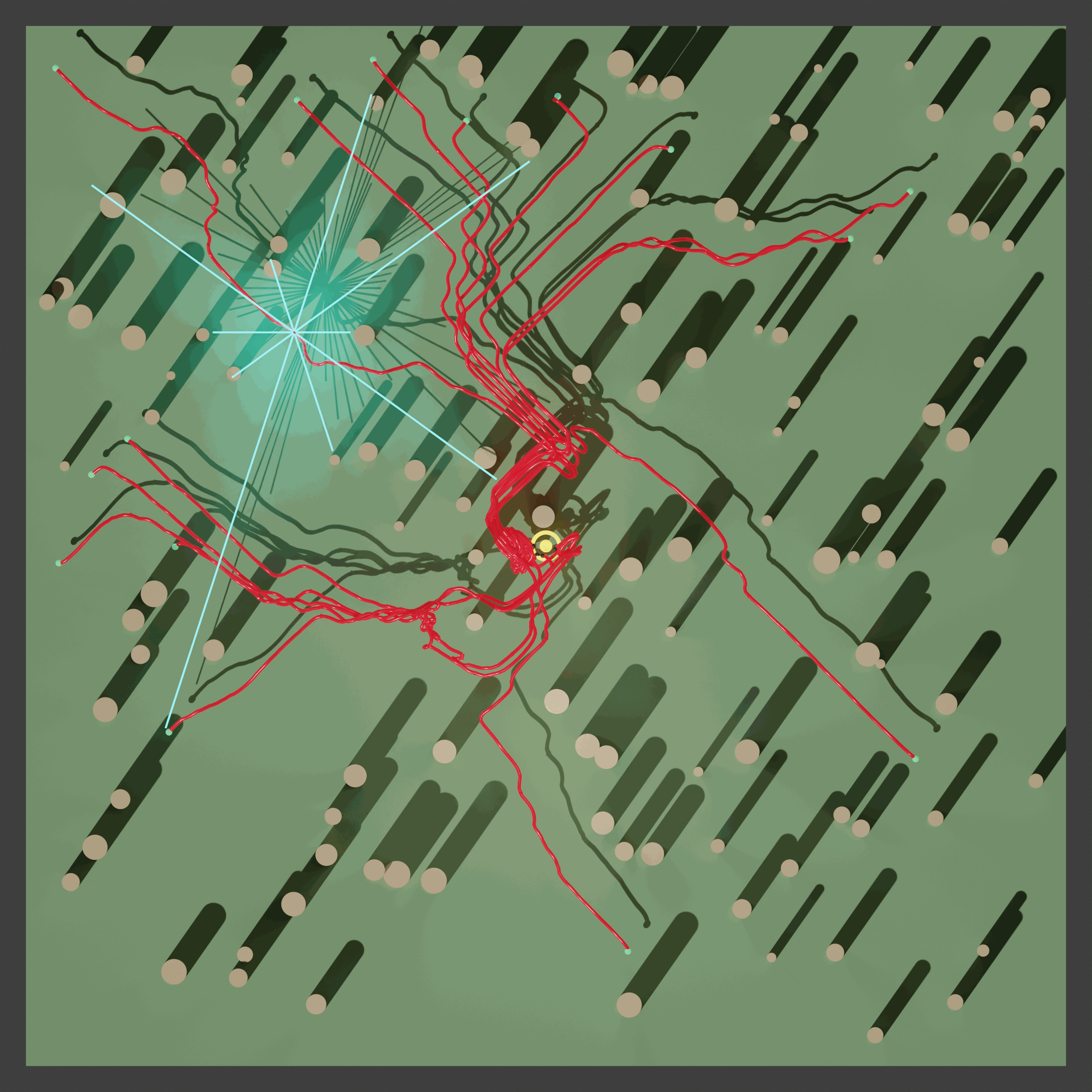}\hfill
\includegraphics[width=0.49\linewidth]{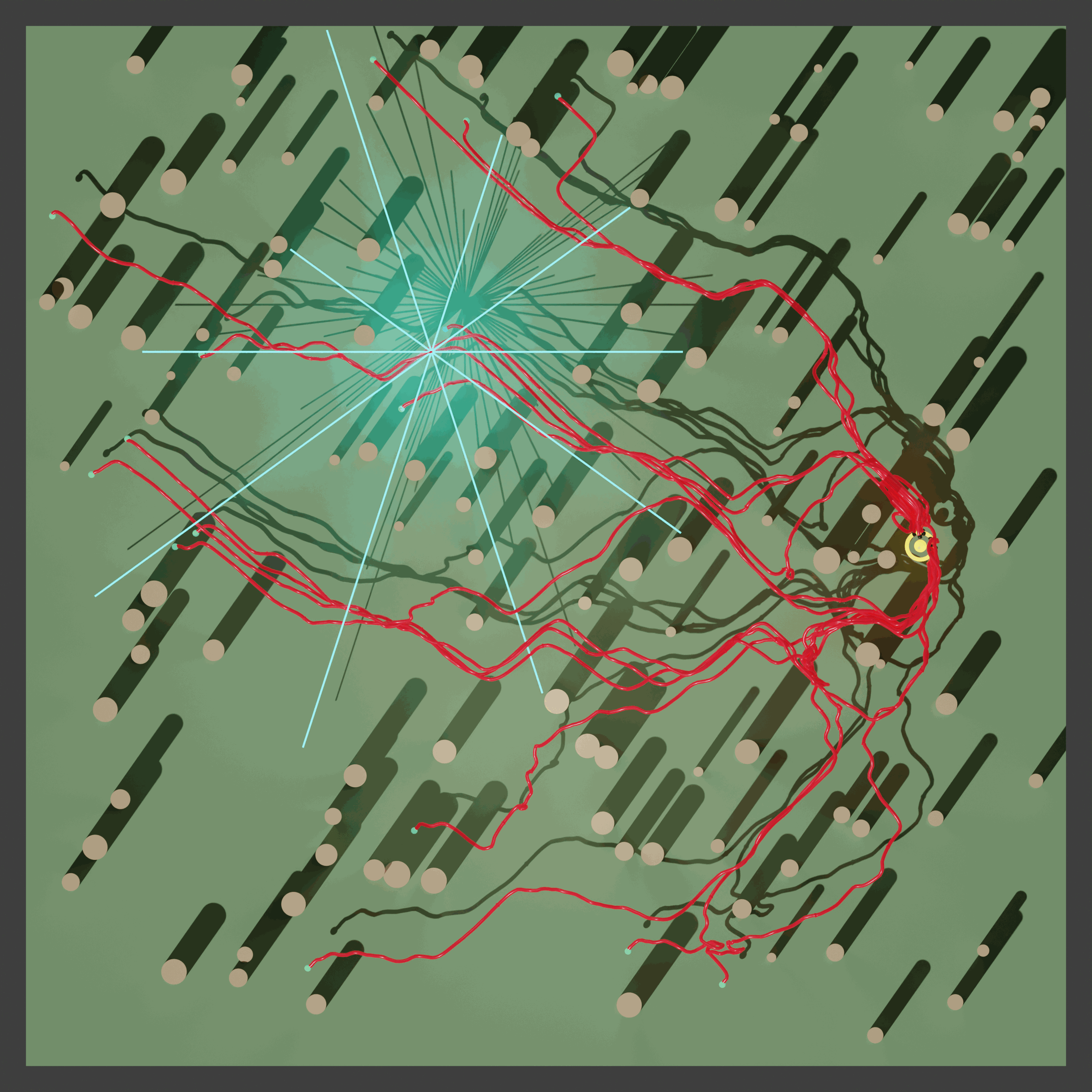}
\par\vspace{-0.45em}
{\scriptsize\makebox[0.49\linewidth][c]{(a)}\hfill\makebox[0.49\linewidth][c]{(b)}}
\par\vspace{0.05em}
\includegraphics[width=0.49\linewidth]{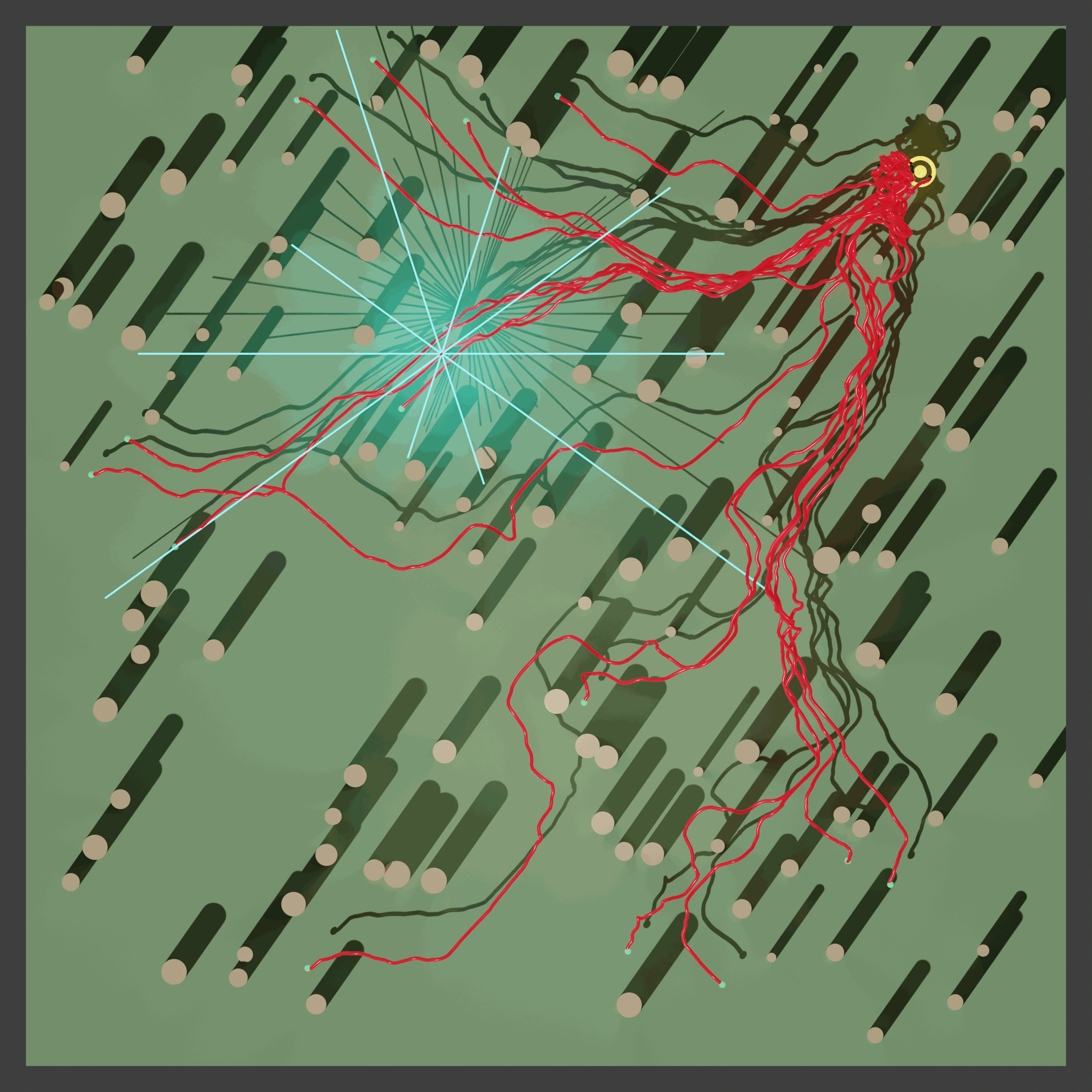}\hfill
\includegraphics[width=0.49\linewidth]{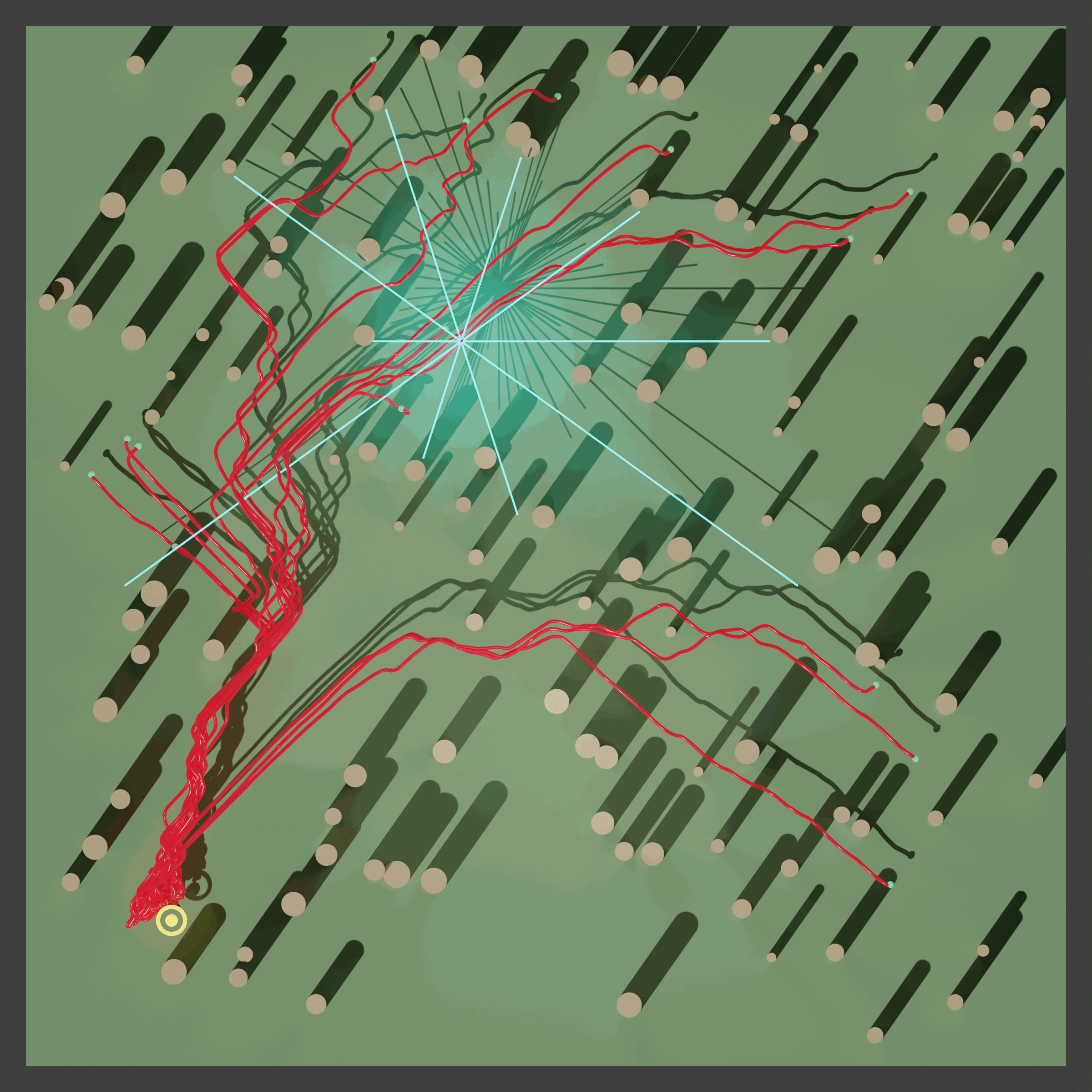}
\par\vspace{-0.45em}
{\scriptsize\makebox[0.49\linewidth][c]{(c)}\hfill\makebox[0.49\linewidth][c]{(d)}}
\end{minipage}
\caption{Top-down trajectories for four Forest targets: (a) \((0,0,2)\), (b) \((18,0,2)\), (c) \((18,18,2)\), and (d) \((-18,-18,2)\). Each panel shows the 15 highest-return trajectories in one representative shared scene.}
\label{fig:uav_forest_trajectory}
\end{figure*}

The aggregate results in Fig.~\ref{fig:uav_forest_summary} use all 240 episodes. Koopman Dreamer attains the highest success rate (73.8\%) and lowest failure rate (26.2\%), compared with 53.8\%/46.2\% for DreamerV3 and 15.8\%/84.2\% for D4PG. Its median final distance is also lowest (0.72 versus 0.75 and 4.22). The two world models have comparable mean returns, while Koopman Dreamer achieves the higher median return (195.50 versus 182.19). Complete statistics are reported in Table~\ref{tab:appendix_uav_forest_eval}. Overall, the structured world model yields the clearest advantage in task completion, which depends on maintaining useful target and obstacle information over long imagined horizons.

\begin{figure*}[!t]
\centering
\includegraphics[width=0.97\textwidth]{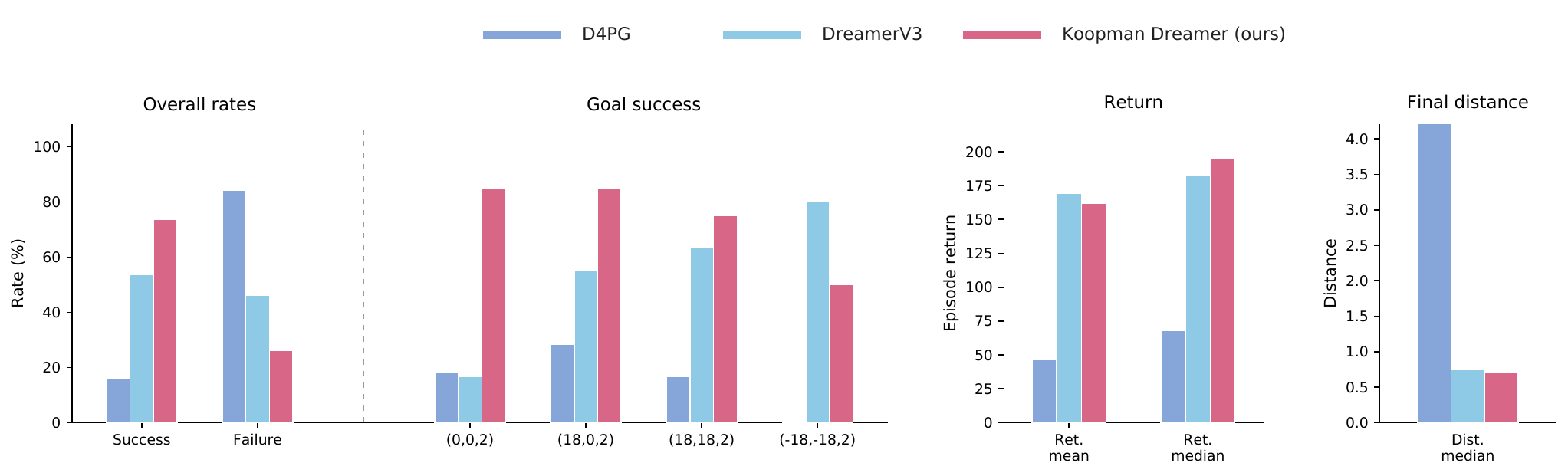}
\caption{Closed-loop UAV-LiDAR results over four targets and three shared Forest scenes.}
\label{fig:uav_forest_summary}
\end{figure*}

\textbf{Open-loop prediction.}
DreamerV3 and Koopman Dreamer are initialized from posterior context and rolled out along recorded actions without future observations. At horizon 64, Koopman Dreamer lowers vector, LiDAR, velocity, and deterministic-latent MSE from 0.0181, 0.0183, 0.00297, and 0.0465 to 0.0160, 0.0161, 0.00174, and 0.0233, respectively. It also reduces mean reward MSE over the rollout from 478.31 to 281.74 and improves target-relative prediction over most evaluated horizons.

Fig.~\ref{fig:uav_openloop} shows the strongest and most persistent reductions in velocity, reward, and deterministic-state drift---quantities that directly affect imagined returns and future control. The spectral blocks limit autonomous amplification, while the linear and bilinear action terms preserve global and state-dependent control effects. The agreement between these open-loop improvements and the higher completion rate provides complementary evidence that the structured transition produces more useful imagined trajectories.

\begin{figure*}[!t]
\centering
\includegraphics[width=0.92\textwidth]{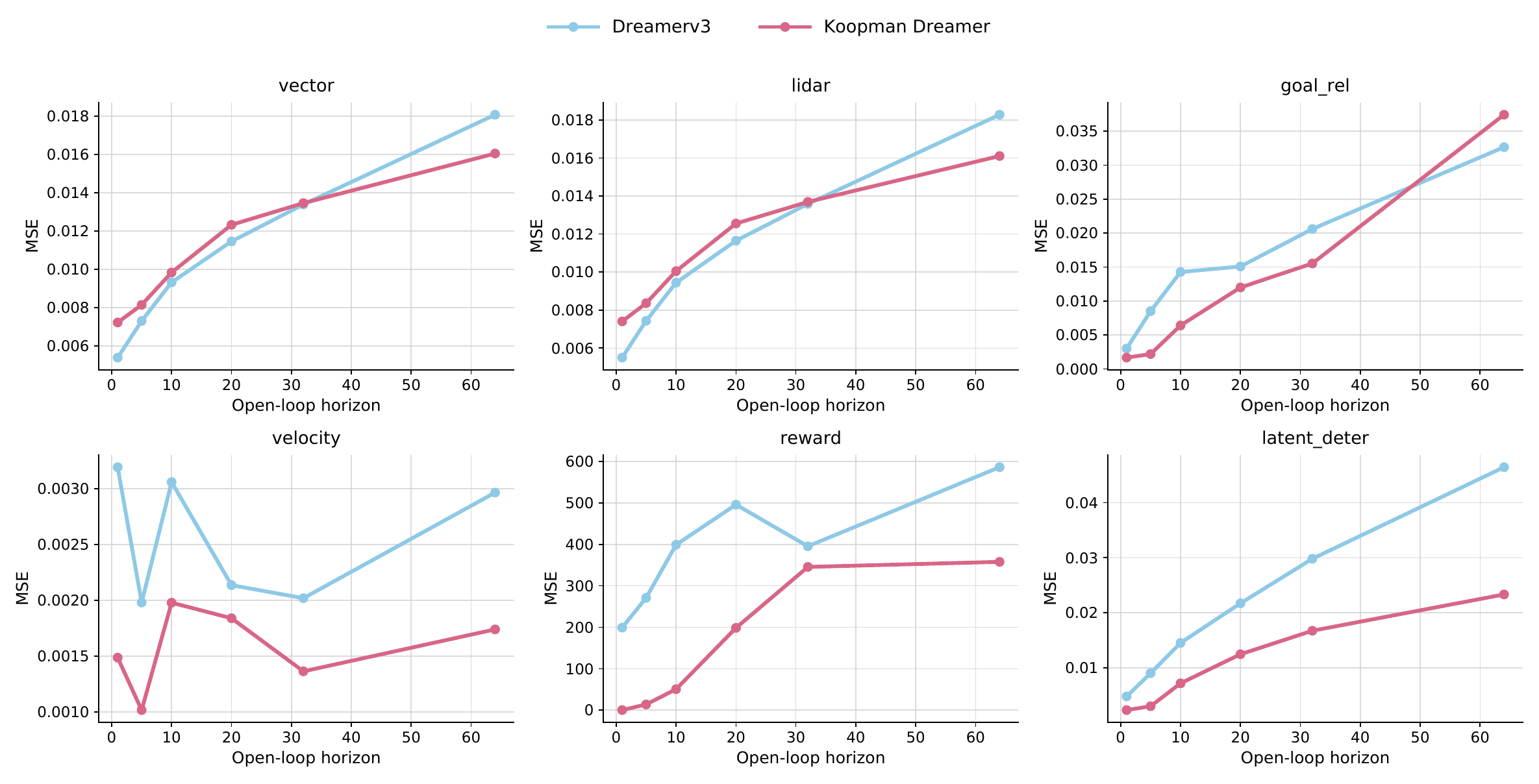}
\caption{Open-loop prediction MSE versus horizon in the UAV-LiDAR Forest scenario.}
\label{fig:uav_openloop}
\end{figure*}

\subsection{Spectral Radius Sensitivity}
\label{c.-spectral-radius-and-long-horizon-prediction-error}

We train six otherwise matched variants with \(\rho_{\min}=0.75\) and different spectral upper bounds. Table~\ref{tab:learned_spectral_radius_config} reports the resulting learned radii. After a 32-step context, each model performs a 64-step open-loop rollout along recorded actions; complete horizon curves are provided in Fig.~\ref{fig:learned_radius} in the appendix.

\begin{table}[!t]
\caption{Spectral Upper-Bound Configurations and Learned Spectral Radii}
\label{tab:learned_spectral_radius_config}
\centering
\footnotesize
\setlength{\tabcolsep}{5.2pt}
\renewcommand{\arraystretch}{1.12}
\begin{tabular}{@{}ccccccc@{}}
\toprule
\(\rho_{\max}\) & 0.90 & 0.95 & 1.00 & 1.05 & 1.10 & 1.20 \\
\(\rho(A_K)\) & 0.849 & 0.881 & 0.912 & 0.943 & 0.971 & 1.026 \\
\bottomrule
\end{tabular}
\renewcommand{\arraystretch}{1.0}
\end{table}

Fig.~\ref{fig:radius_sensitivity_h64} confirms the stability--expressiveness trade-off predicted by \eqref{eq:theory_multistep}. The smallest radius, \(0.849\), gives the lowest latent MSE (\(0.0140\)) but not the lowest observable errors, indicating that strong contraction can make latent trajectories numerically similar while discarding predictive information. Observable prediction is strongest at intermediate radii: \(0.881\) minimizes vector, LiDAR, and velocity MSE, whereas \(0.943\) minimizes goal-relative and reward MSE.

As the learned radius approaches and crosses the unit circle, accumulated error becomes dominant. Increasing \(\rho_\star\) from \(0.943\) to \(1.026\) raises horizon-64 latent MSE from \(0.0518\) to \(2.2675\), with corresponding degradation in vector and LiDAR prediction. Thus, the spectral constraint is not simply a contraction mechanism: it provides a learnable range that preserves persistent modes without allowing uncontrolled long-horizon amplification. Because latent coordinates differ across variants, observable metrics remain the primary basis for selecting this range.

\begin{figure*}[!t]
\centering
\includegraphics[width=0.98\textwidth]{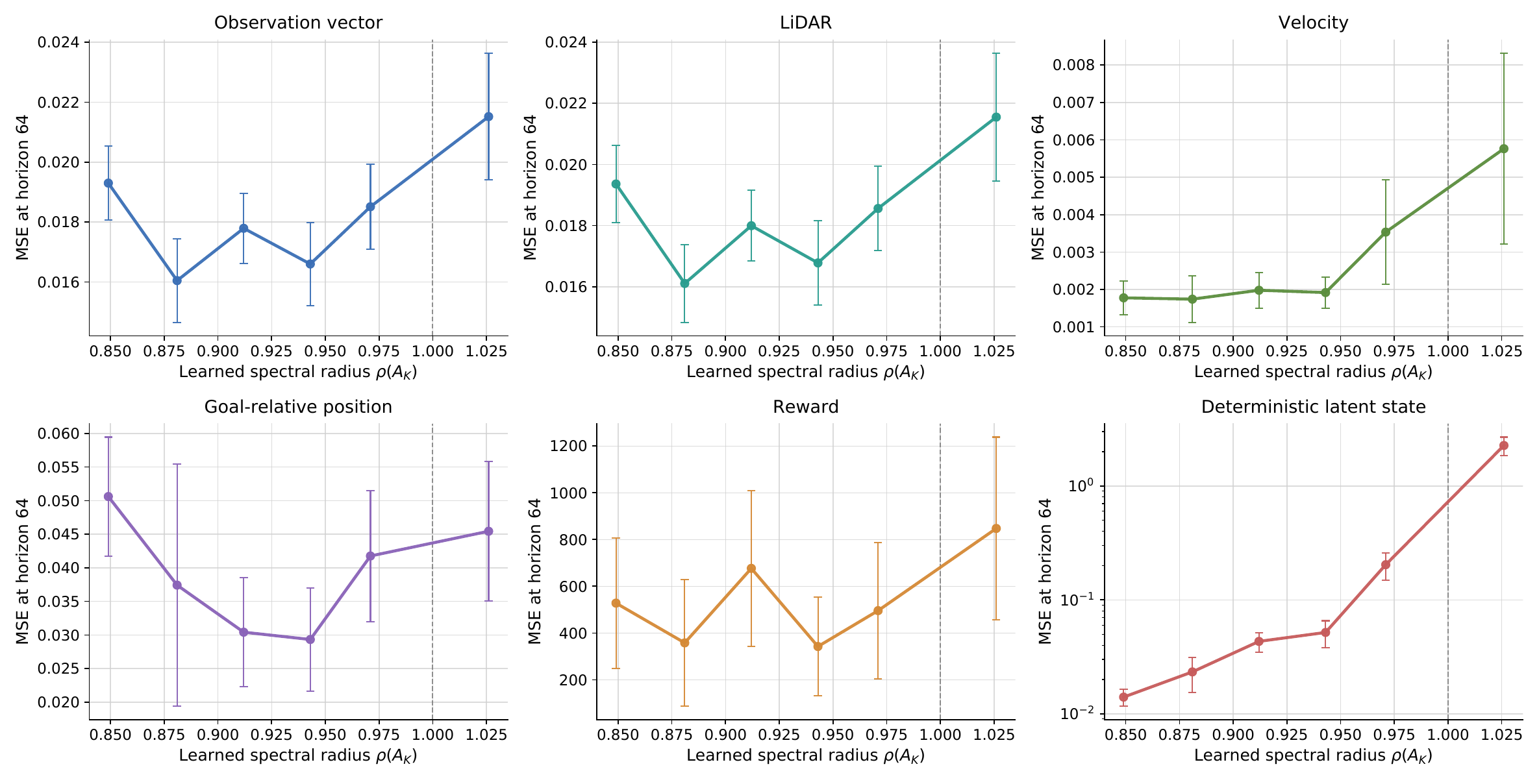}
\caption{Horizon-64 prediction MSE as a function of the learned spectral radius. Error bars show variation across evaluation batches.}
\label{fig:radius_sensitivity_h64}
\end{figure*}

\subsection{Structural Ablation}
\label{c.-ablation-study}

We evaluate the full model and three variants for 80 episodes on one shared Forest scene. Table~\ref{tab:uav_ablation_config} lists the exact modifications and their diagnostic roles. All other architecture, loss, training-budget, and evaluation settings are fixed.

\begin{table}[!t]
\caption{Structural Configurations of UAV-LiDAR Ablation Models}
\label{tab:uav_ablation_config}
\centering
\footnotesize
\setlength{\tabcolsep}{2.6pt}
\renewcommand{\arraystretch}{1.12}
\begin{tabular}{@{}>{\raggedright\arraybackslash}p{0.23\columnwidth}
                >{\raggedright\arraybackslash}p{0.43\columnwidth}
                >{\raggedright\arraybackslash}p{0.26\columnwidth}@{}}
\toprule
Variant & Modified component & Diagnostic focus \\
\midrule
Koopman Dreamer & Spectral blocks, posterior-conditioned EMA teacher, and bilinear control retained & Full structured model \\
w/o spectral & Bounded-radius parameterization disabled; modal radii optimized without a radius penalty; operator regularization retained & Long-horizon spectral control \\
w/o teacher & Teacher projection and EMA branches disabled; stop-gradient posterior \(\phi\) and online \(z\)-projection used as targets & Target stability and alignment \\
w/o bilinear & Bilinear term removed; only \(B_a\bar a_t\) retained & State-dependent action effects \\
\bottomrule
\end{tabular}
\renewcommand{\arraystretch}{1.0}
\end{table}

\textbf{Closed-loop effects.}
Fig.~\ref{fig:uav_ablation} assigns distinct roles to the three components. Removing the spectral constraint leaves success comparable but reduces mean return by 20.4\% (163.82 to 130.33) and raises median action variation from 0.97 to 4.00. The spectral structure therefore contributes primarily through rollout quality and smoother, less compensatory control. Removing the teacher lowers mean return to 143.20 and success to 68.8\%, while increasing the collision rate from 15.0\% to 28.7\%. This suggests that slowly updated posterior-conditioned targets improve closed-loop robustness by preventing rapidly changing online targets from weakening task-relevant transition learning. Removing bilinear control is most damaging: success falls to 8.8\% and the collision rate rises to 56.2\%, demonstrating that a fixed linear action direction is insufficient for state-dependent navigation dynamics. Complete termination statistics are reported in Table~\ref{tab:appendix_uav_ablation_eval}.

\begin{figure*}[!t]
\centering
\includegraphics[width=0.98\textwidth]{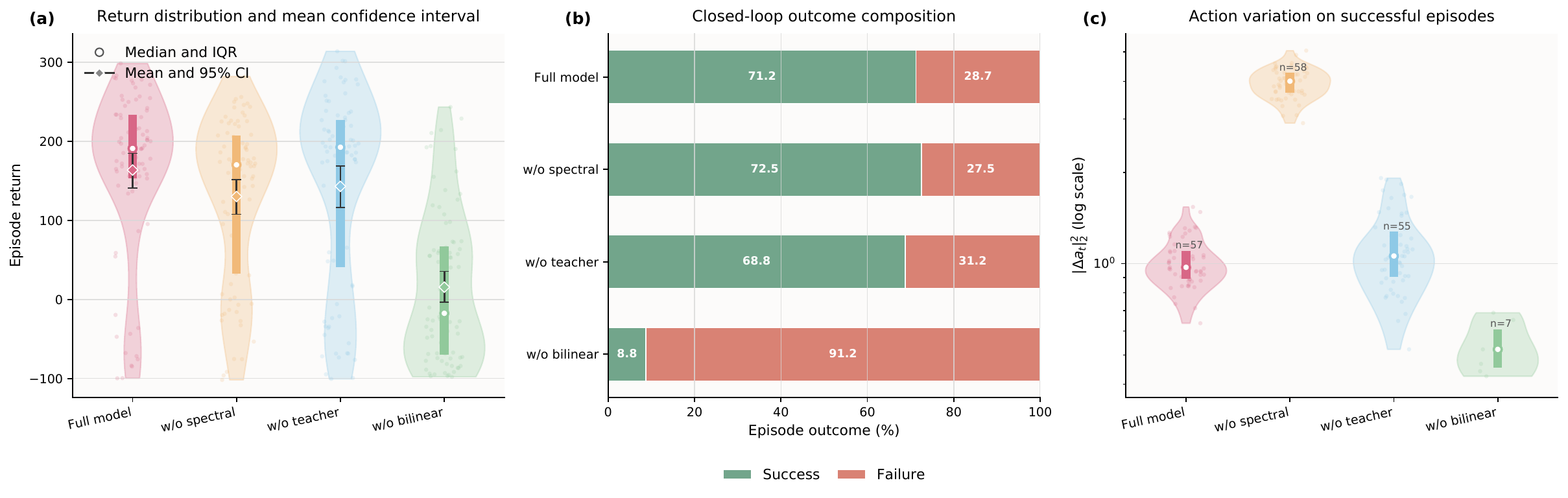}
\caption{Closed-loop ablation results: return distribution, outcome composition, and action variation on successful episodes.}
\label{fig:uav_ablation}
\end{figure*}

\textbf{Open-loop effects.}
Fig.~\ref{fig:uav_ablation_openloop} explains the closed-loop differences at the dynamics level. At horizon 64, removing the spectral constraint increases vector, LiDAR, velocity, and latent MSE from \(0.0160\), \(0.0161\), \(0.00174\), and \(0.0233\) to \(0.0228\), \(0.0229\), \(0.00690\), and \(0.3891\), respectively. This broad error growth directly supports the spectral backbone's role in suppressing long-horizon drift. Removing bilinear control increases goal-relative and reward MSE from \(0.0374\) and \(357.81\) to \(0.0674\) and \(1036.99\), connecting its severe control degradation to inaccurate state-dependent action propagation. The teacher ablation retains similar average observation error but has higher reward error and a markedly higher collision rate, indicating that stable teacher targets contribute information relevant to robust closed-loop decisions beyond average observation reconstruction.

\begin{figure*}[!t]
\centering
\includegraphics[width=0.96\textwidth]{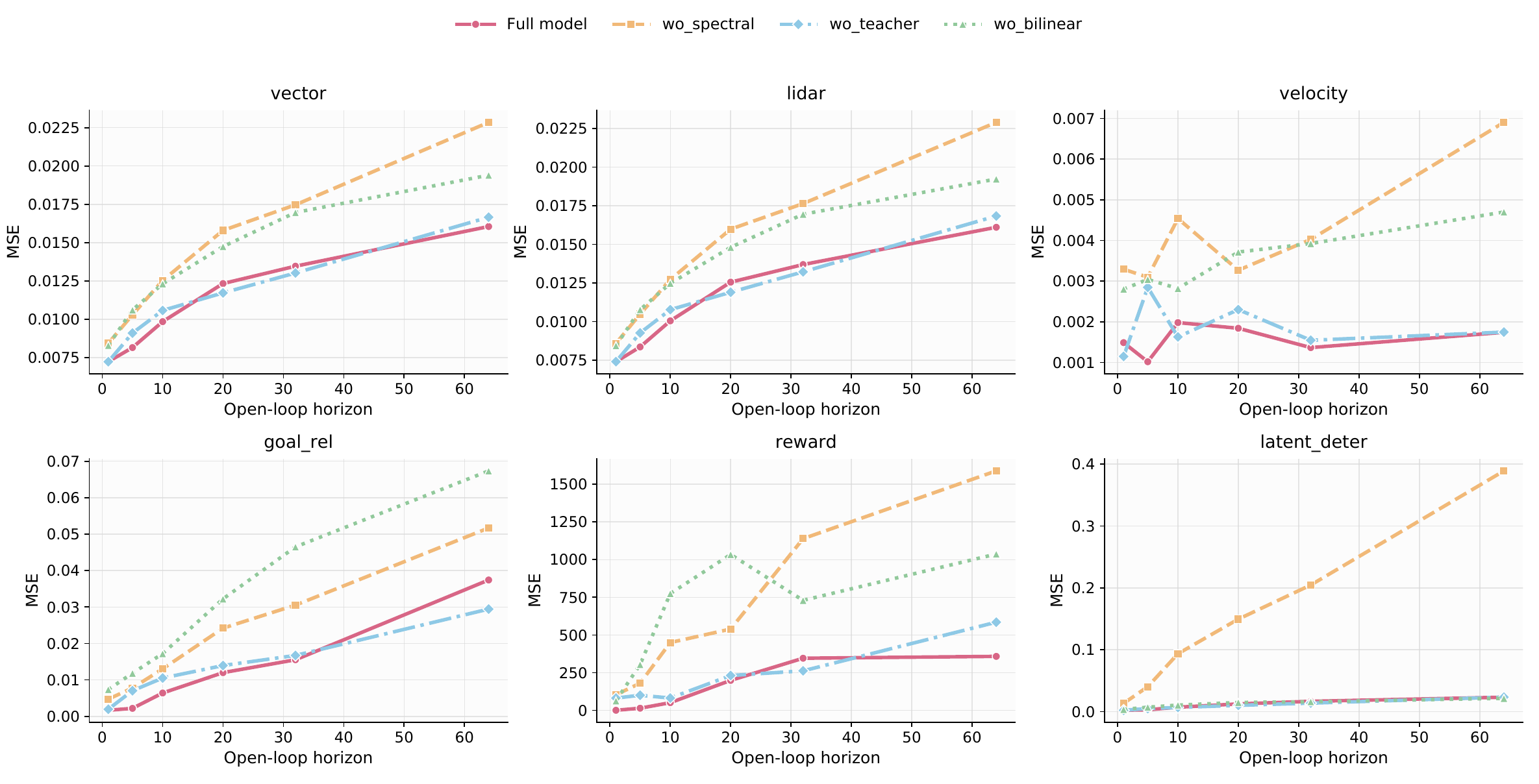}
\caption{Open-loop prediction errors of the ablation models versus horizon. All variants use a 32-step context and 64-step prediction horizon.}
\label{fig:uav_ablation_openloop}
\end{figure*}

Taken together, the closed- and open-loop results support a complementary structural hierarchy. The spectral constraint limits long-horizon drift and the compensatory action variation it induces; posterior-conditioned EMA targets improve representation--transition alignment and closed-loop robustness; and bilinear control captures the state-dependent action effects required for nonlinear navigation. These components jointly turn the spectrally organized latent backbone into an effective controlled world model for long-horizon imagination. The ablations are conducted in a representative Forest setting, and broader scene-level evaluation can further examine the generality of these roles.
\section{Discussion}\label{vi.-discussion}
The experimental results support the central claim that organizing deterministic latent propagation around a spectrally constrained Koopman backbone can improve long-horizon imagination in Dreamer-style world models. On the nine DMC proprioceptive tasks, Koopman Dreamer outperforms the DreamerV3 world-model baseline on eight tasks and achieves the best final score among the evaluated methods on six. The open-loop results provide complementary evidence: in addition to substantially reducing deterministic-latent rollout error, Koopman Dreamer lowers decoded proprioceptive prediction error on most tasks. This consistency between open-loop prediction and closed-loop control suggests that the improvements on swing-up, sustained motion, precise reaching, and balance are related to more reliable multi-step latent propagation rather than to task-specific capacity alone.

The spectral-radius sensitivity results further clarify the role of the proposed structure. Strong contraction can suppress persistent task-relevant information, whereas overly persistent modes increase the influence of accumulated transition errors. The favorable intermediate spectral range therefore reflects a stability--expressiveness trade-off rather than a preference for the smallest possible radius. Within this structure, the spectral backbone organizes the persistence, attenuation, and oscillation of autonomous latent modes, while posterior-conditioned teacher targets and rollout objectives reduce stepwise prediction residuals. The linear and bilinear action terms further adapt the backbone to controlled dynamics by representing both globally consistent and state-dependent action effects. The ablation results support these complementary roles and indicate that the auxiliary mechanisms strengthen, rather than replace, the spectral dynamics core.

The results span compact proprioceptive observations and higher-dimensional LiDAR-based vectors. In UAV-LiDAR navigation, Koopman Dreamer improves target completion and substantially reduces the overall failure rate relative to DreamerV3. Overall, the experiments support spectrally structured latent dynamics as an effective inductive bias for long-horizon prediction and control in the evaluated settings. Future work may extend the approach with adaptive spectral ranges, broader evaluation across training runs and environments, and safety-aware or risk-sensitive policy objectives.
\section{Conclusion}\label{vii.-conclusion}
This paper proposed Koopman Dreamer for end-to-end continuous control. Bounded rotation--scaling blocks explicitly characterize the damping, persistence, and oscillation of autonomous latent modes, while linear and bilinear action terms and stochastic modulation adapt this spectral backbone to controlled nonlinear dynamics. Posterior-conditioned EMA targets and prior-rollout objectives train the same structured transition under both observation-corrected representation learning and posterior-free imagination. The resulting error-propagation analysis shows that reliable long-horizon prediction requires a balance between error attenuation and persistent information retention rather than maximal contraction. Experiments on DMC proprioceptive control and UAV-LiDAR navigation, together with open-loop prediction, spectral-radius sensitivity analysis, and structural ablations, demonstrate that the structure of Koopman Dreamer improves long-horizon latent propagation and supports stronger closed-loop task performance. Overall, Koopman Dreamer provides a structured and controllable dynamics foundation for long-horizon imagination in Dreamer-style world models. Future work may investigate adaptive spectral ranges and combine Koopman Dreamer with safety-aware or risk-sensitive policy objectives for more reliable autonomous control.


%
\clearpage
\onecolumn
\appendix[Experimental Details and Supplementary Results]
\label{appendix-a.-experimental-configuration-and-supplementary-material}
\subsection{DMC Proprioceptive Benchmark Details}
The PPO, D4PG, DDPG, MPO, DMPO, and DreamerV3 curves used to compute Table~\ref{tab:dmc_last50} are taken from the public DeepMind Control Suite proprioceptive benchmark reported in the DreamerV3 paper \cite{hafner2025dreamerv3}. This benchmark uses continuous actions, proprioceptive vector inputs, and a budget of 500K environment steps. Control Suite tasks use an action repeat of 2. For consistency with the main evaluation protocol in this paper, the reported final score of each baseline seed is computed from the last recorded point in the corresponding public JSON curve, and the table reports the mean over seeds. The DreamerV3 paper reports benchmark results under fixed hyperparameter settings and provides a unified training protocol in its supplementary material. This paper does not repeat the internal hyperparameters of these public baselines and uses them only as published comparison results. Koopman Dreamer is implemented and trained in this paper, with its main configuration shown in Table~\ref{tab:appendix_dmc_config}.
\begin{table}[!t]
\caption{Main Hyperparameters of Koopman Dreamer on DMC Proprioceptive Tasks}
\label{tab:appendix_dmc_config}
\centering
\footnotesize
\setlength{\tabcolsep}{4pt}

\begin{tabular}{@{}p{0.18\textwidth}p{0.26\textwidth}p{0.47\textwidth}@{}}
\toprule
Category & Hyperparameter & Value \\
\midrule
Environment and training & Task set & 9 DMC proprioceptive tasks, env.dmc.image=False \\
Environment and training & Interaction budget & 500K environment steps \\
Environment and training & Batch and updates & batch size 16, batch length 64, train ratio 512 \\
World-model scale & preset & size12m \\
World-model scale & Koopman hidden & deterministic state 2048, hidden size 256, stochastic 32 groups \(\times\) 16 classes \\
World-model scale & Koopman representation & \(\phi_t\) dimension 2048, \(z_t\) projection dimension 64 \\
Network width & encoder / decoder / heads & depth 16, units 256 \\
Policy optimization & actor--critic imagination & imagination length 15, discount horizon 333, return lambda 0.95 \\
Koopman backbone & Spectral constraint & maximum radius 0.95, minimum radius 0.85, initial radius 0.90 \\
Koopman backbone & Control and stochastic modulation & \(z\) scale 0.08, dense scale 0.0, teacher EMA rate 0.01 \\
Auxiliary losses & loss scales & koop 0.02, roll 0.02, opreg 0.002, pred 0.05 \\
Auxiliary losses & warmup & pred 20K, roll 50K, koop 100K, opreg 150K \\
\bottomrule
\end{tabular}
\end{table}
\begin{table}[!t]
\caption{Sources of Public DMC Proprioceptive Comparison Results}
\label{tab:appendix_acme_config}
\centering
\footnotesize
\setlength{\tabcolsep}{4pt}

\begin{tabular}{@{}p{0.18\textwidth}p{0.38\textwidth}p{0.35\textwidth}@{}}
\toprule
Method & Source and protocol & Description \\
\midrule
PPO & Public Proprio Control benchmark from the DreamerV3 paper & PPO reference with fixed hyperparameters \\
D4PG / DDPG / MPO / DMPO & Public Proprio Control benchmark from the DreamerV3 paper & Public continuous-control comparison results \\
DreamerV3 & Public Proprio Control benchmark from the DreamerV3 paper & Public world-model baseline result \\
Koopman Dreamer & Implemented and trained in this paper & Uses the configuration in Table~\ref{tab:appendix_dmc_config} \\
\bottomrule
\end{tabular}
\end{table}

\subsection{UAV-LiDAR Configuration and Complete Tables}
The UAV experiments use gym\_UAVLidarNav-v0, with observation key vector and action key action. The main closed-loop navigation evaluation, trajectory visualization, open-loop comparison between DreamerV3 and Koopman Dreamer, spectral-radius sensitivity analysis, and structural ablations are all conducted in the Forest scenario, using fixed targets and shared initial conditions as specified in the corresponding experiment. D4PG, DreamerV3, and Koopman Dreamer use the same Forest environment and evaluation protocol for the main closed-loop comparison. Table~\ref{tab:uav_scenes} summarizes the UAV-LiDAR scenario variants, Table~\ref{tab:appendix_uav_config} gives the Forest environment and evaluation configuration used in the main experiments, and Table~\ref{tab:appendix_uav_method_config} lists the method hyperparameters.
\begin{table}[!t]
\caption{UAV-LiDAR Forest Autonomous Navigation Environment and Evaluation Configuration}
\label{tab:appendix_uav_config}
\centering
\footnotesize
\setlength{\tabcolsep}{4pt}

\begin{tabular}{@{}p{0.18\textwidth}p{0.26\textwidth}p{0.47\textwidth}@{}}
\toprule
Category & Configuration item & Value \\
\midrule
Environment & Gymnasium task & gym\_UAVLidarNav-v0, observation key vector, action key action \\
Environment & Forest autonomous navigation environment & world size 50 m, 120 cylindrical obstacles, dynamic obstacles disabled \\
Environment & LiDAR & 280 dimensions, 7 rings, max range 20 m, proximity encoding, clearance feature enabled \\
Environment & Dynamics and episode & time step 0.1 s, maximum horizontal speed 2.0 m/s, maximum vertical speed 1.0 m/s, max steps 500 \\
Evaluation & Fixed targets & \((0,0,2)\), \((18,0,2)\), \((18,18,2)\), \((-18,-18,2)\) \\
Evaluation & Main-comparison scenes & 3 shared forest scenes \\
Evaluation & Main-comparison episodes & 20 per target per scene, 80 per scene, 240 per method in total \\
Evaluation & Ablation episodes & 1 representative shared scene, 20 per target, 80 per variant in total \\
\bottomrule
\end{tabular}
\end{table}

\begin{table}[!t]
\caption{Main Hyperparameters of Methods on UAV-LiDAR Forest Autonomous Navigation}
\label{tab:appendix_uav_method_config}
\centering
\footnotesize
\setlength{\tabcolsep}{4pt}

\begin{tabular}{@{}p{0.15\textwidth}p{0.20\textwidth}p{0.56\textwidth}@{}}
\toprule
Method & Category & Value \\
\midrule
D4PG & Training and sampling & 500K agent decision steps, action repeat 1, local actors 16, evaluation episodes 10 \\
D4PG & actor / critic & actor MLP \(256,256,256\), critic MLP \(256,256,256\), distributional critic atoms 51 \\
D4PG & replay and batch & batch size 256, min replay size 1000, max replay size \(10^6\), samples per insert 0 \\
D4PG & optimization and target & learning rate \(3\times 10^{-4}\), Gaussian exploration standard deviation 0.2, discount 0.99, 5-step return, critic support \([-150,150]\) \\
DreamerV3 & Training and sampling & 500K environment steps, 24 parallel environments, evaluation environments 4, evaluation episodes 4, train ratio 128 \\
DreamerV3 & World-model scale & size25m, deterministic state 3072, hidden size 384, stochastic 32 groups \(\times\) 24 classes \\
DreamerV3 & Batch and policy optimization & batch size 16, batch length 64, imagination length 15, return lambda 0.95 \\
Koopman Dreamer & Training and sampling & 500K environment steps, 24 parallel environments, evaluation environments 4, evaluation episodes 4, train ratio 128 \\
Koopman Dreamer & World-model scale & size25m, \(\phi_t\) dimension 3072, hidden size 384, stochastic 32 groups \(\times\) 24 classes, \(z_t\) projection dimension 96 \\
Koopman Dreamer & Spectral and control & maximum radius 0.95, minimum radius 0.75, initial radius 0.87, learned radius 0.881, control scale 3.0 \\
Koopman Dreamer & Bilinear control & bilinear scale 0.05, bilinear rank 256, bilinear outscale 0.01 \\
Koopman Dreamer & Stochastic modulation and teacher & \(z\) scale 0.10, dense scale 0.0, teacher EMA rate 0.01 \\
Koopman Dreamer & Auxiliary losses & koop 0.005, roll 0.005, opreg 0.001, pred 0.01; warmups are 150K, 100K, 200K, and 60K, respectively \\
\bottomrule
\end{tabular}
\end{table}

\begin{table}[!t]
\caption{Complete Numerical Results for the Four-Target UAV-LiDAR Forest Autonomous Navigation Evaluation Averaged over Three Shared Evaluation Scenes}
\label{tab:appendix_uav_forest_eval}
\centering
\scriptsize
\setlength{\tabcolsep}{4pt}

\begin{tabular}{@{}lrrrrrrrr@{}}
\toprule
Method & Episodes & Return mean & Return median & Success & Failure & Length mean & Path mean & Dist median \\
\midrule
D4PG & 240 & 46.52 & 67.81 & 15.8\% & 84.2\% & 386.00 & 72.88 & 4.22 \\
DreamerV3 & 240 & \textbf{169.21} & 182.19 & 53.8\% & 46.2\% & 359.76 & 60.09 & 0.75 \\
\shortstack[l]{Koopman\\Dreamer} & 240 & 161.98 & \textbf{195.50} & \textbf{73.8\%} & \textbf{26.2\%} & \textbf{194.70} & \textbf{40.45} & \textbf{0.72} \\
\bottomrule
\end{tabular}
\end{table}

\begin{table}[!t]
\caption{Complete Ablation Results on UAV-LiDAR Forest Autonomous Navigation}
\label{tab:appendix_uav_ablation_eval}
\centering
\scriptsize
\setlength{\tabcolsep}{4.5pt}
\begin{tabular*}{\textwidth}{@{\extracolsep{\fill}}lrrrrrr@{}}
\toprule
Variant & Episodes & Return mean [95\% CI] & Success & Collision & Timeout &
\shortstack{Successful-trajectory\\action variation\\median [IQR]} \\
\midrule
\shortstack[l]{Koopman\\Dreamer} & 80 & 163.82 [140.85, 184.88] & 57/80 (71.2\%) & 15.0\% & 13.8\% & 0.97 [0.89, 1.10] \\
w/o spectral & 80 & 130.33 [107.80, 151.85] & 58/80 (72.5\%) & 16.2\% & 11.2\% & 4.00 [3.67, 4.26] \\
w/o teacher & 80 & 143.20 [116.25, 169.06] & 55/80 (68.8\%) & 28.7\% & 2.5\% & 1.06 [0.90, 1.27] \\
w/o bilinear & 80 & 15.64 [-3.57, 35.78] & 7/80 (8.8\%) & 56.2\% & 35.0\% & 0.52 [0.45, 0.61] \\
\bottomrule
\end{tabular*}
\par\vspace{0.35em}
\parbox{\textwidth}{\scriptsize CI denotes confidence interval, and IQR denotes interquartile range. Return CIs are obtained by stratified bootstrap resampling within each of the four target sets. Action variation is the temporal mean of the squared Euclidean difference between consecutive actions on successful trajectories.}
\end{table}
\subsection{Supplementary Spectral-Radius Horizon Curves}
Fig.~\ref{fig:learned_radius} complements the horizon-64 summary in Fig.~\ref{fig:radius_sensitivity_h64} by showing how prediction errors evolve across all evaluated open-loop horizons.
\begin{figure}[!t]
\centering
\includegraphics[width=0.98\textwidth]{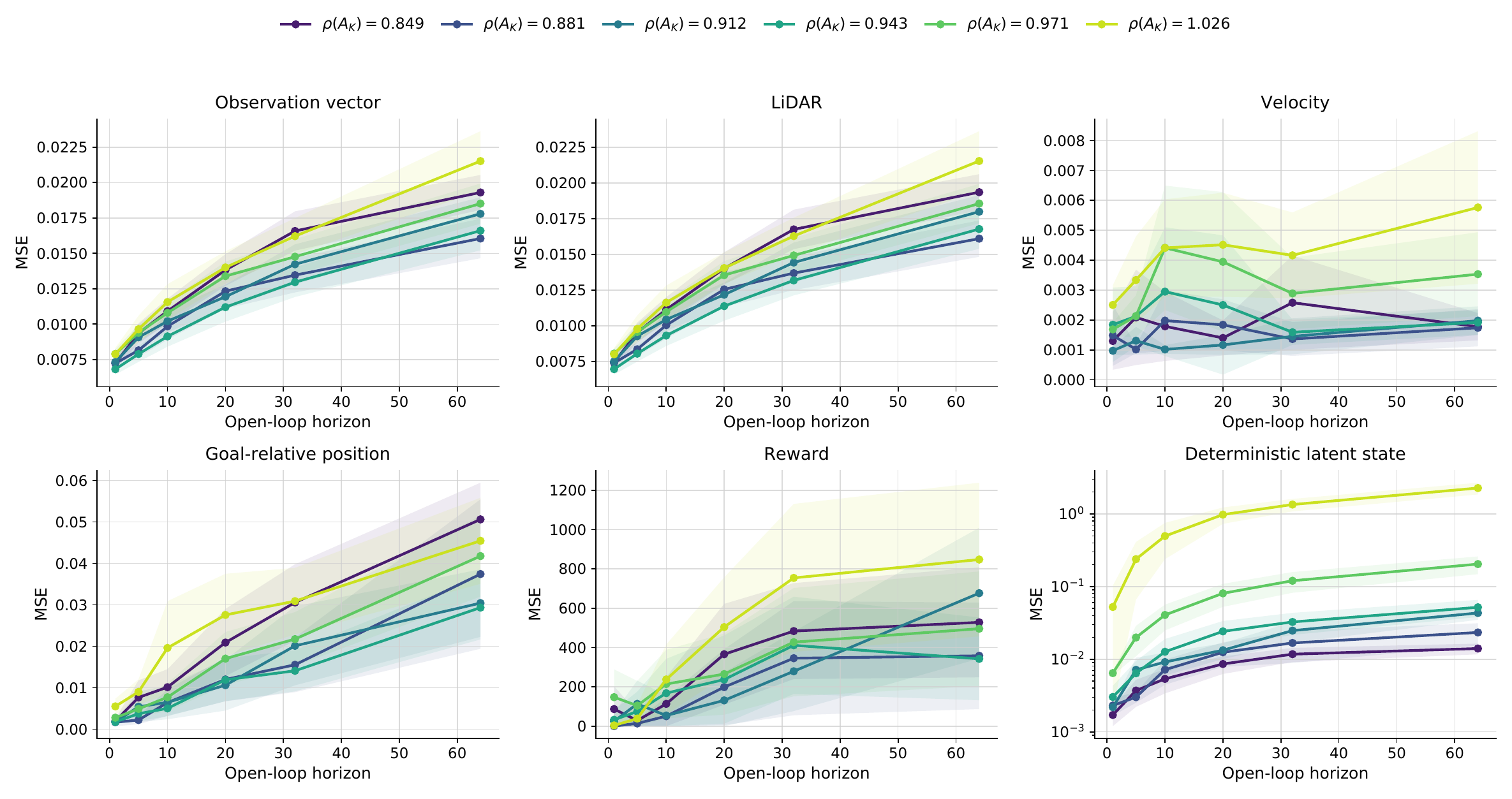}
\caption{Open-loop prediction error as a function of prediction horizon under different learned spectral radii in the UAV-LiDAR Forest scenario. The six panels report mean-squared errors for the complete observation vector, LiDAR observations, velocity, target-relative position, reward, and deterministic latent state. Solid lines indicate means over evaluation batches, and shaded regions indicate approximate 95\% confidence intervals of batch means.}
\label{fig:learned_radius}
\end{figure}
\subsection{Supplementary DMC Open-Loop Curves}
Fig.~\ref{fig:appendix_dmc_openloop_all_1} and Fig.~\ref{fig:appendix_dmc_openloop_all_2} provide open-loop prediction curves for all nine DMC proprioceptive tasks. Each evaluation uses a 32-step observation context and a horizon-64 open-loop rollout along recorded action sequences, with 32 replay batches and 16 sequences per batch. The three common metrics summarized in Table~\ref{tab:dmc_openloop_summary} are reported at horizons 1, 5, 10, 20, 32, and 64; where applicable, the task figures additionally show errors for task-specific state components. Overall, Koopman Dreamer gives the most consistent improvement in deterministic-latent prediction, with lower horizon-64 error on all nine tasks, and it lowers proprioceptive and reward prediction error on most tasks. These results support the role of the spectrally constrained backbone in improving long-horizon latent propagation and task-relevant prediction.
\begin{figure}[!t]
\centering
\subfloat[{\scriptsize Acrobot Swingup}]{\includegraphics[width=0.49\textwidth]{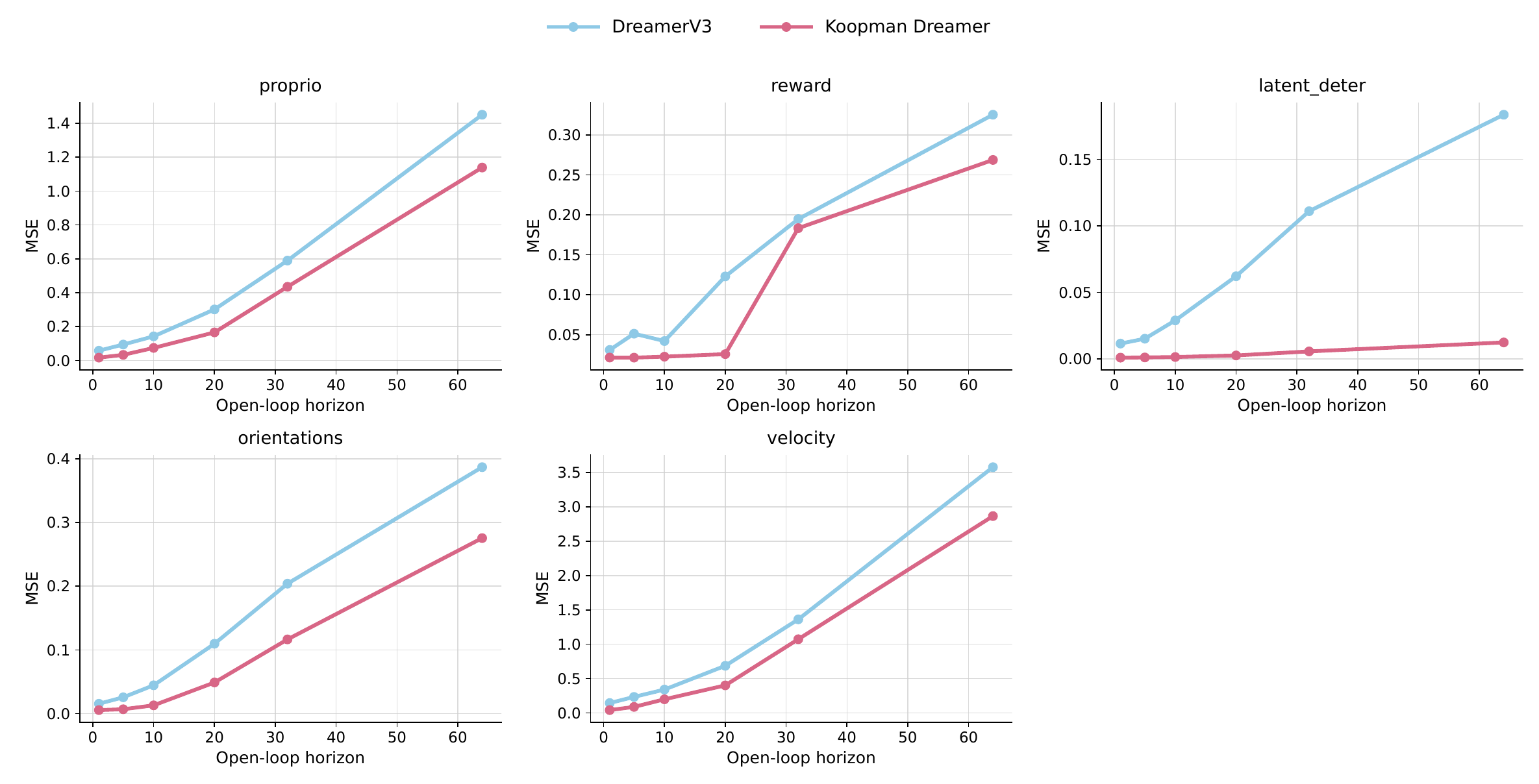}}
\hfil
\subfloat[{\scriptsize Cartpole Swingup}]{\includegraphics[width=0.49\textwidth]{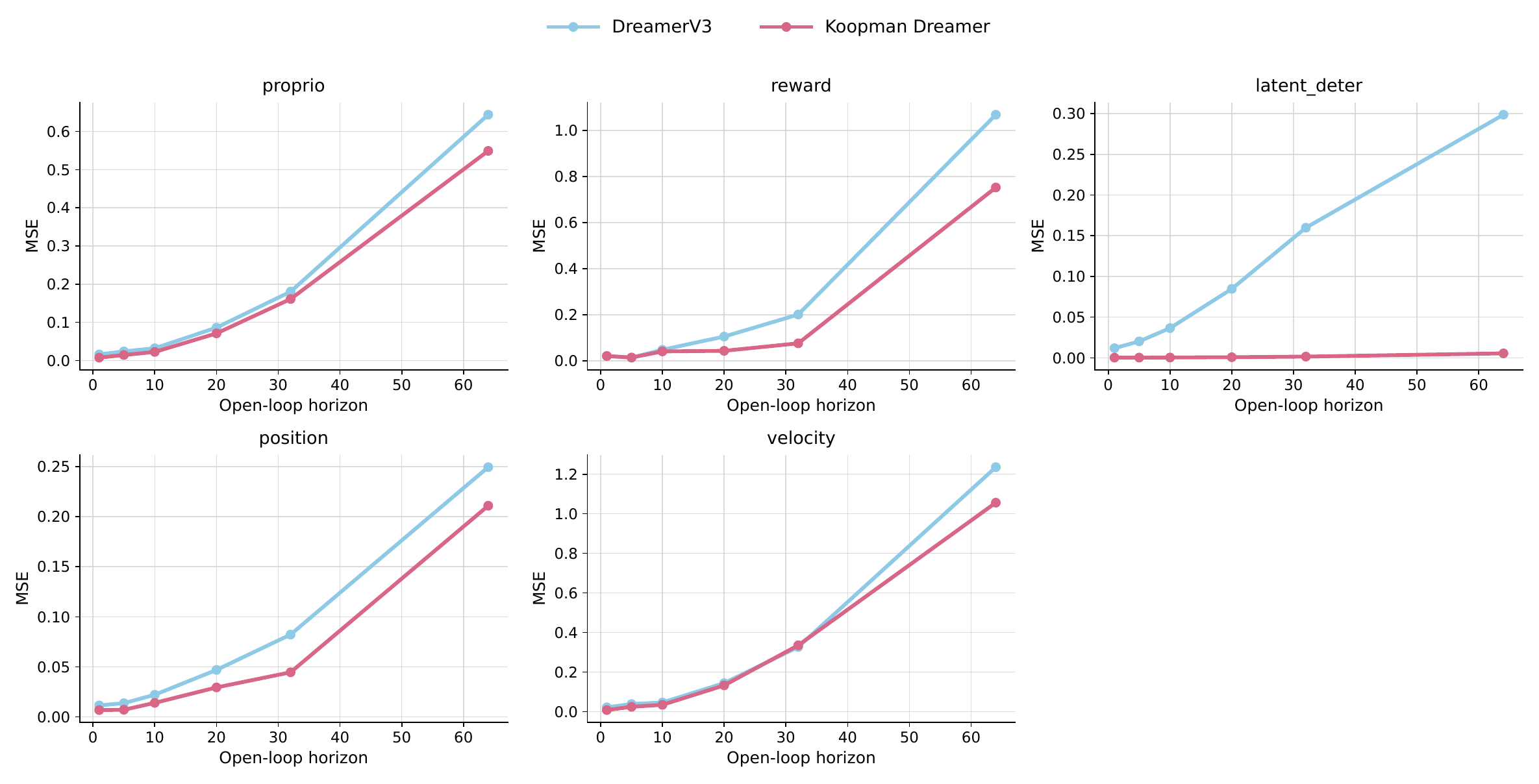}}
\par\vspace{0.35em}
\subfloat[{\scriptsize Cheetah Run}]{\includegraphics[width=0.49\textwidth]{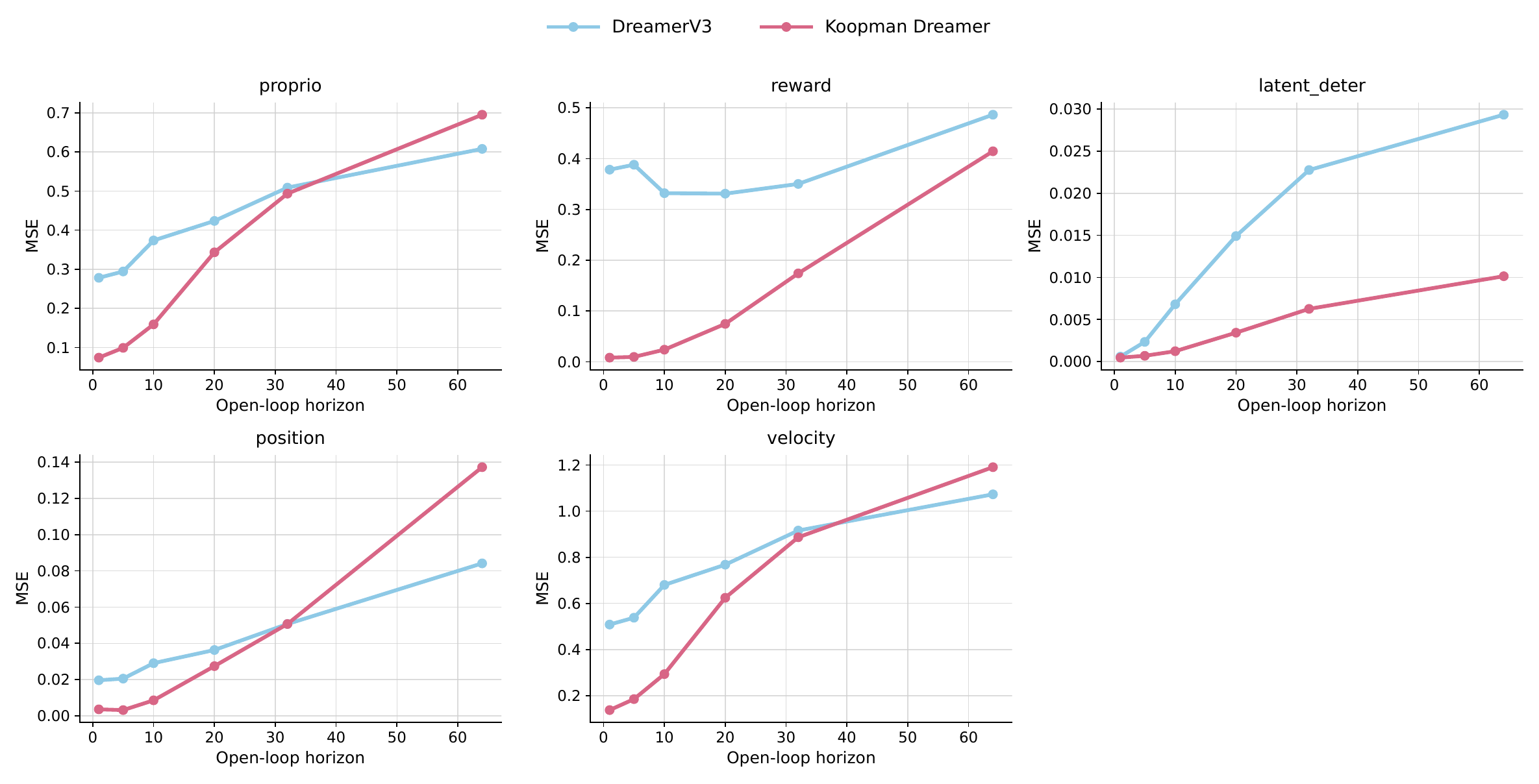}}
\hfil
\subfloat[{\scriptsize Hopper Stand}]{\includegraphics[width=0.49\textwidth]{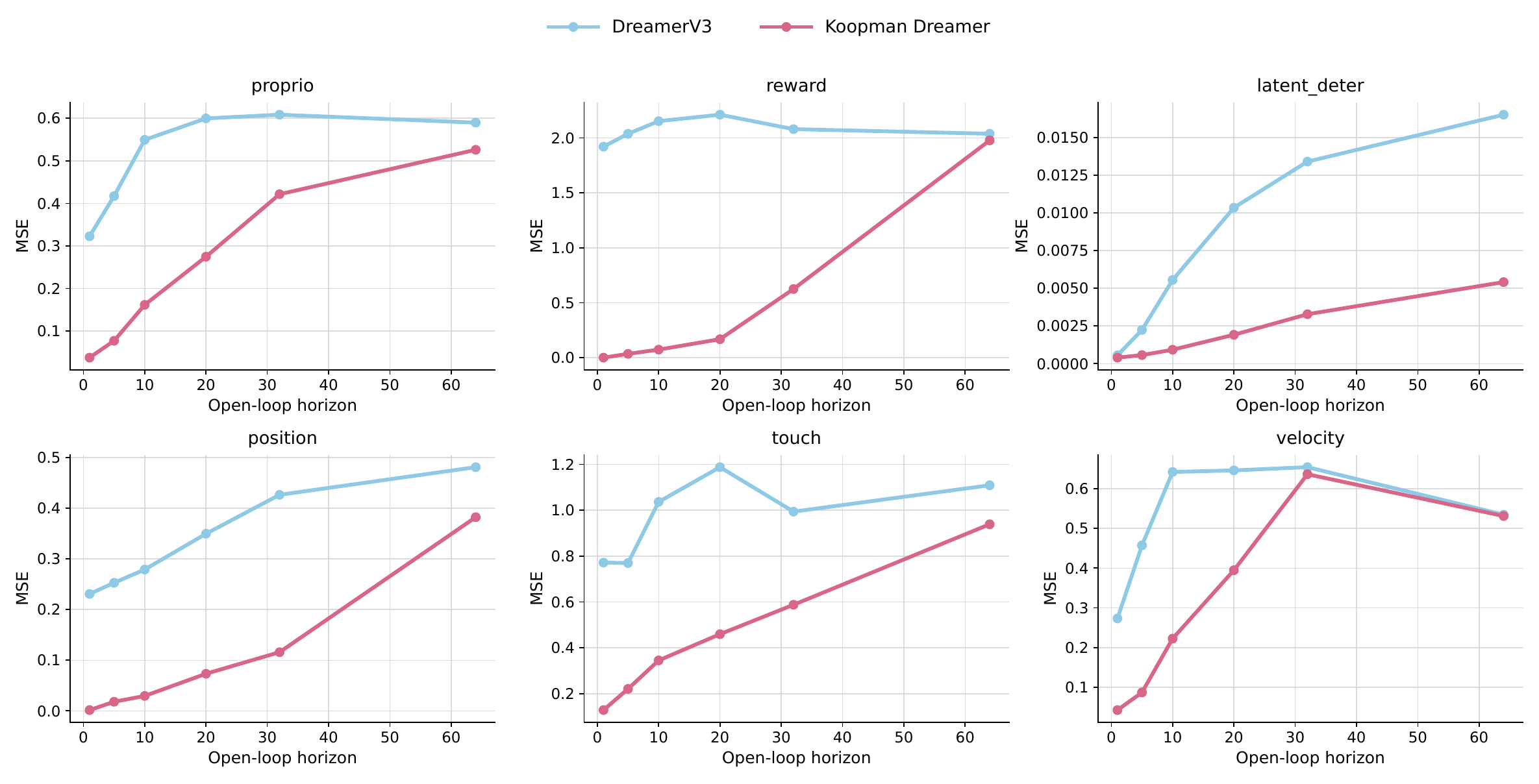}}
\par\vspace{0.35em}
\subfloat[{\scriptsize Reacher Easy}]{\includegraphics[width=0.49\textwidth]{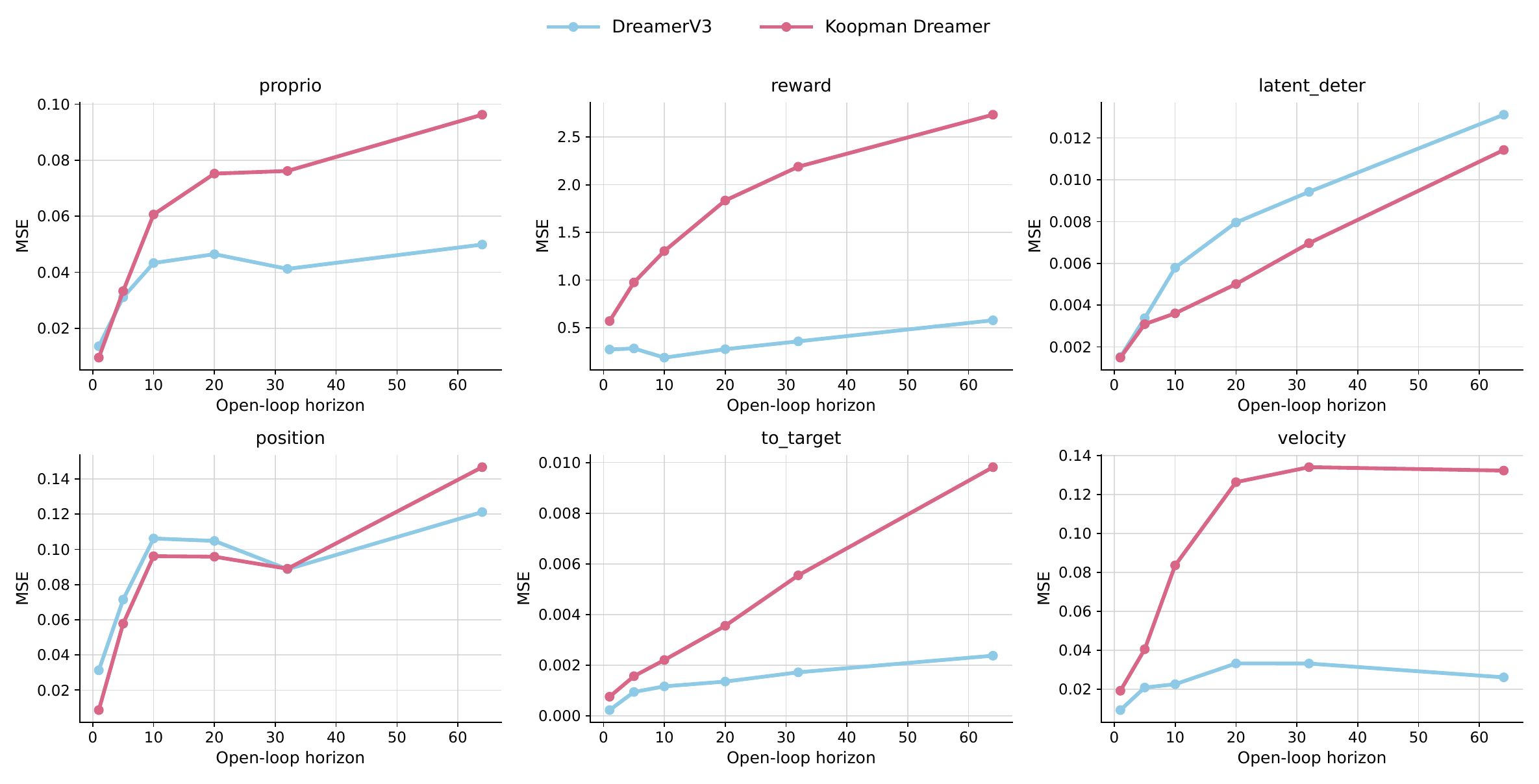}}
\hfil
\subfloat[{\scriptsize Reacher Hard}]{\includegraphics[width=0.49\textwidth]{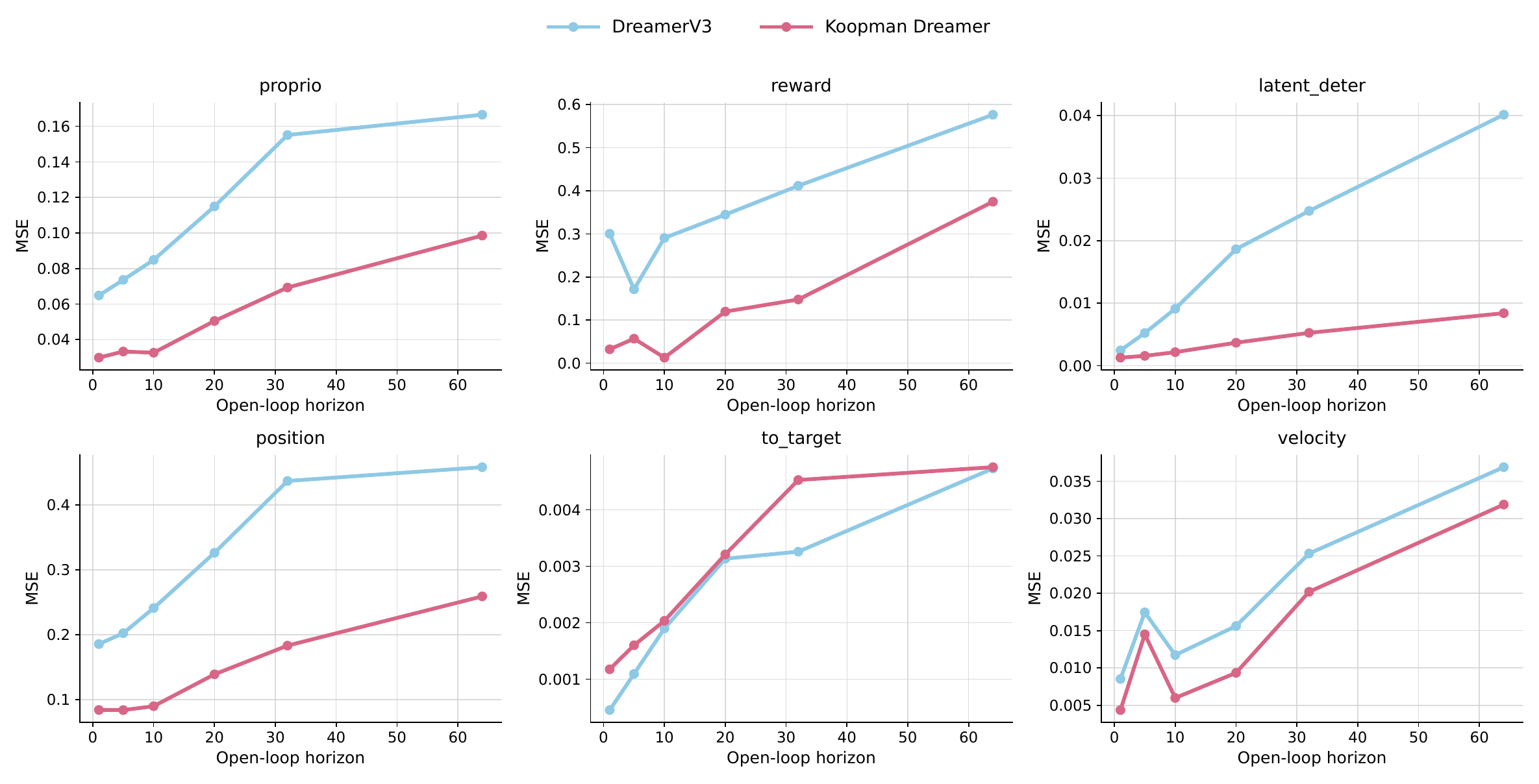}}
\caption{Complete open-loop prediction curves on DMC proprioceptive tasks, part 1.}
\label{fig:appendix_dmc_openloop_all_1}
\end{figure}
\begin{figure}[!t]
\centering
\subfloat[{\scriptsize Walker Run}]{\includegraphics[width=0.49\textwidth]{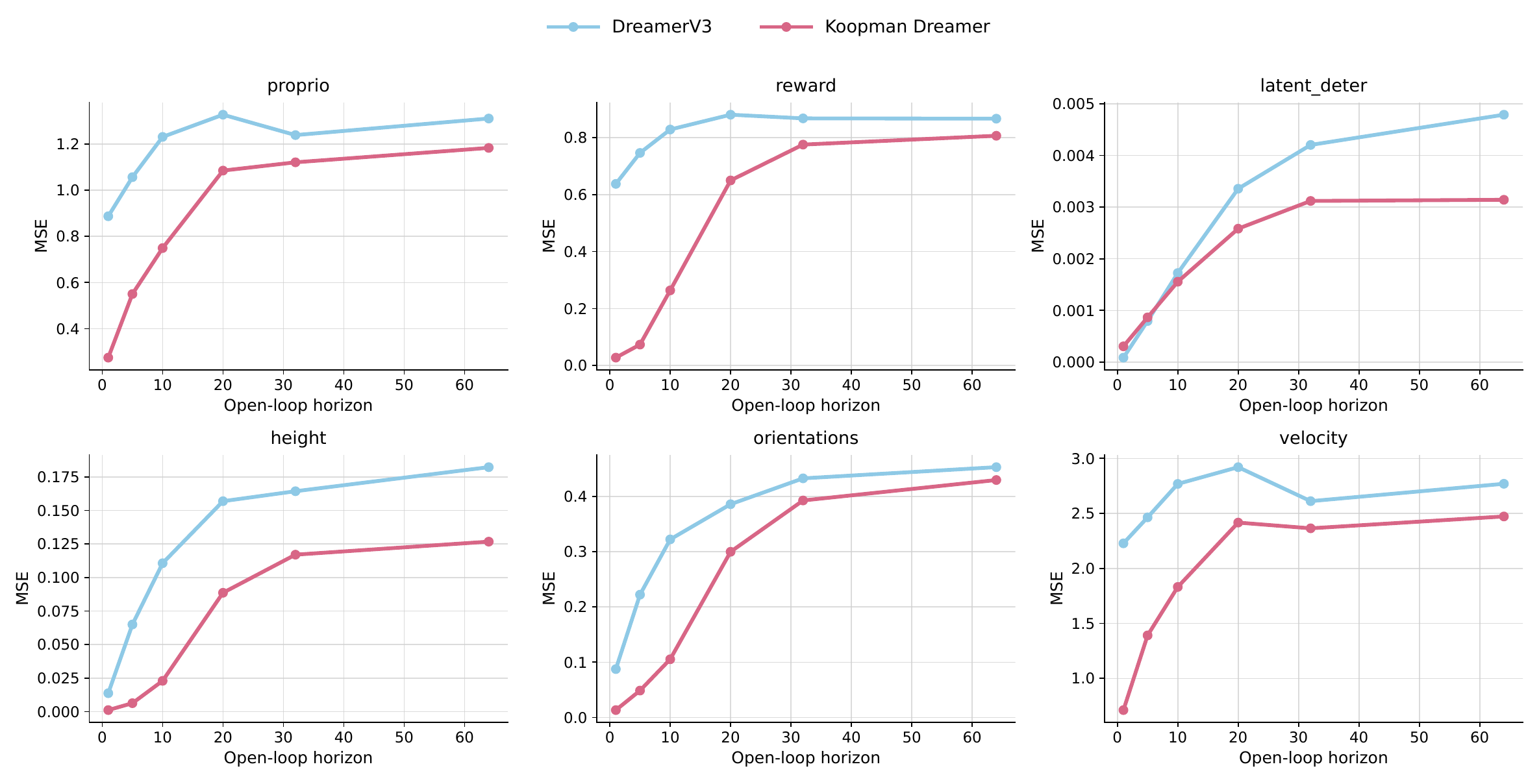}}
\hfil
\subfloat[{\scriptsize Walker Stand}]{\includegraphics[width=0.49\textwidth]{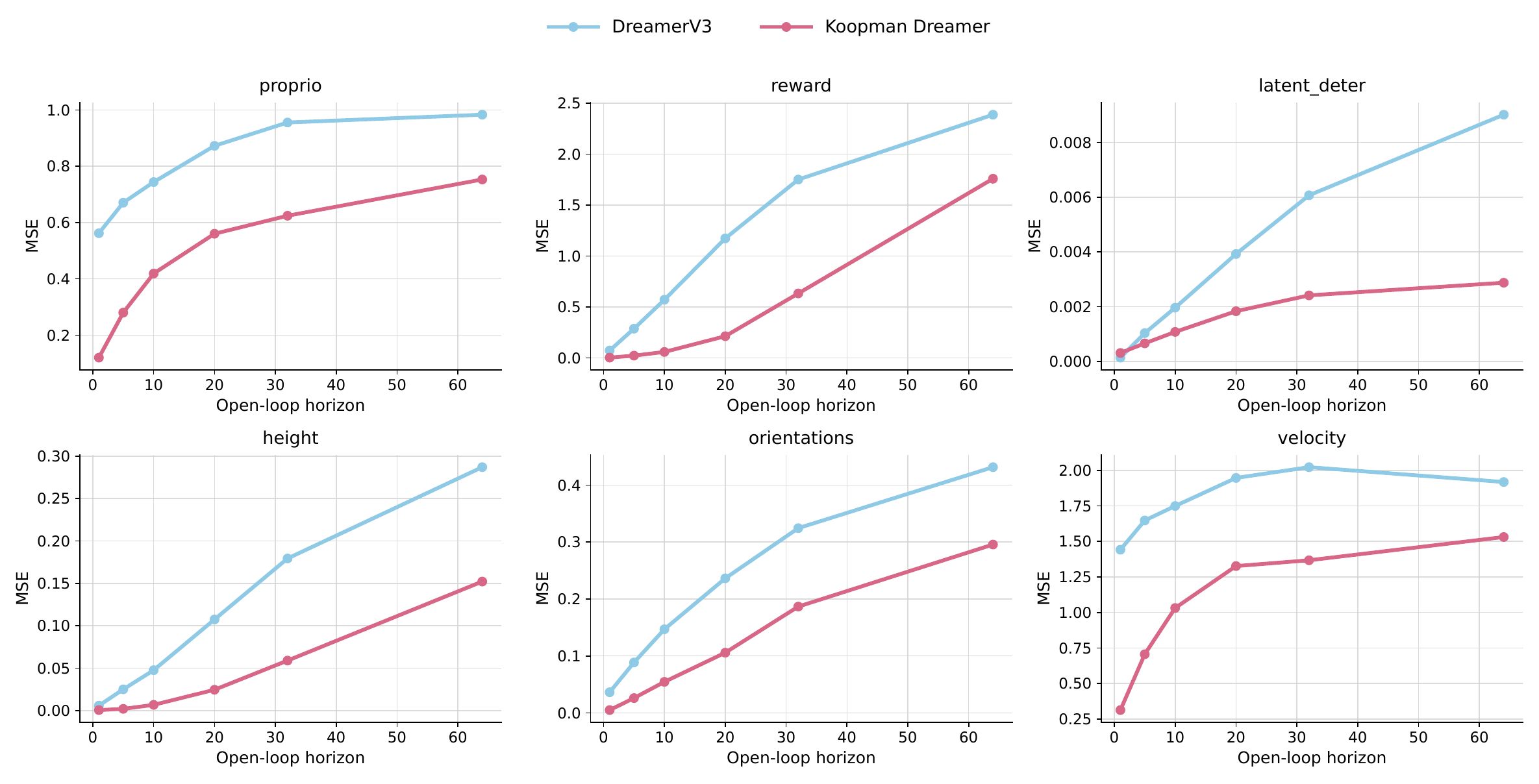}}
\par\vspace{0.35em}
\subfloat[{\scriptsize Walker Walk}]{\includegraphics[width=0.49\textwidth]{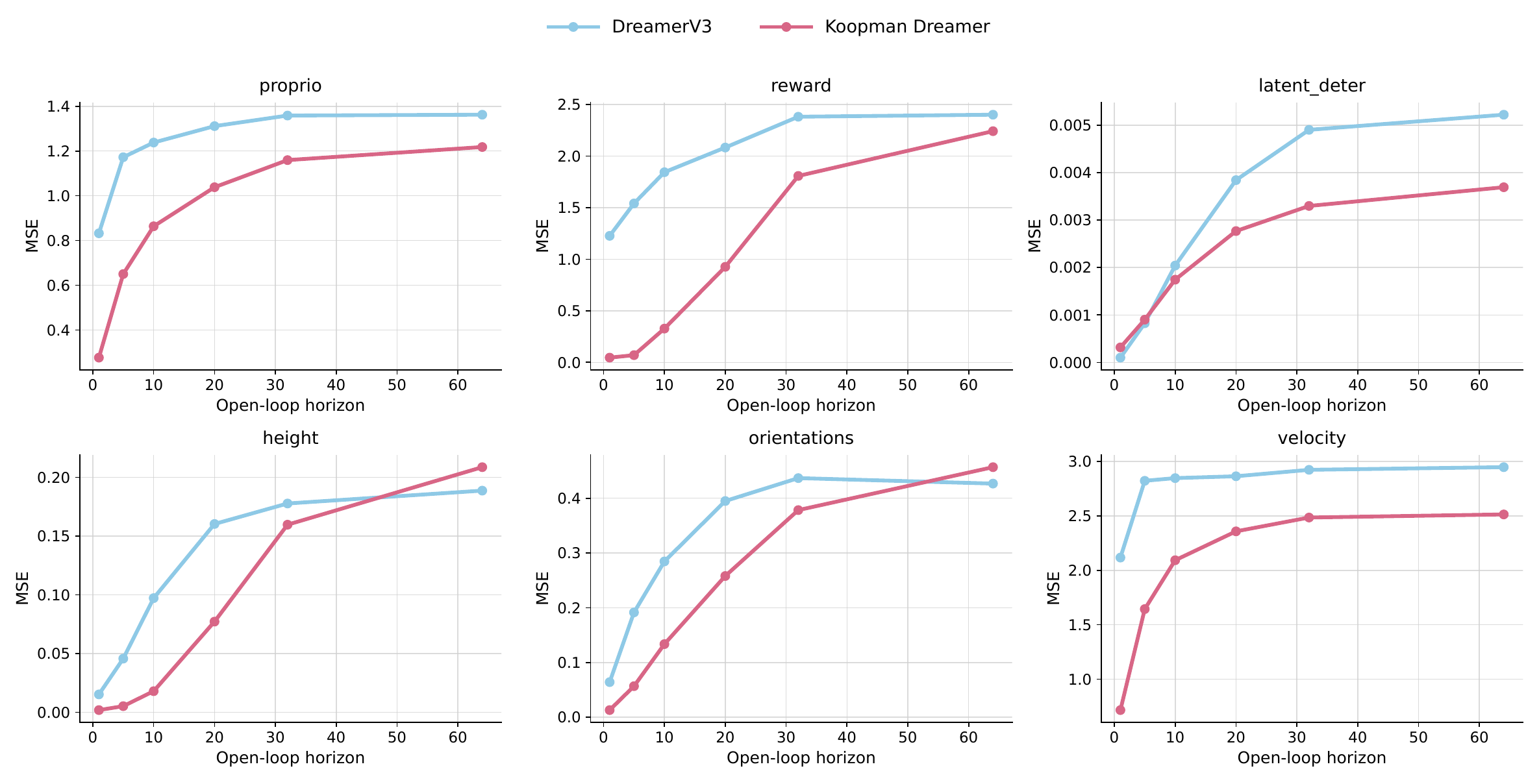}}
\caption{Complete open-loop prediction curves on DMC proprioceptive tasks, part 2.}
\label{fig:appendix_dmc_openloop_all_2}
\end{figure}

\end{document}